\def\year{2022}\relax
\documentclass[letterpaper]{article} 
\usepackage{aaai22}  
\usepackage{times}  
\usepackage{helvet}  
\usepackage{courier}  
\usepackage[hyphens]{url}  
\usepackage{graphicx} 
\usepackage{natbib}  
\usepackage{caption} 
\DeclareCaptionStyle{ruled}{labelfont=normalfont,labelsep=colon,strut=off} 
\usepackage{algorithm}
\usepackage{algorithmic}

\usepackage{newfloat}
\usepackage{listings}
\floatstyle{ruled}
\newfloat{listing}{tb}{lst}{}
\floatname{listing}{Listing}
\usepackage{booktabs}       
\usepackage{tabularx}
\usepackage{amsmath, amsfonts, amssymb}
\usepackage{xcolor}
\usepackage{enumerate}
\usepackage{multirow}
\usepackage{subcaption}
\usepackage{xspace}

\makeatletter
\DeclareRobustCommand\onedot{\futurelet\@let@token\@onedot}
\def\@onedot{\ifx\@let@token.\else.\null\fi\xspace}

\def\eg{\emph{e.g}\onedot} 
\def\ie{\emph{i.e}\onedot} 
\def\cf{\emph{c.f}\onedot}

\makeatother

\newcommand{\tr}{\mbox{$^{\top}$}}

\def\R{{\rm I} \! {\rm R}}

\newcommand{\eq}[1]{Eq~\ref{eq:#1}}

\newcommand{\fig}[1]{Fig~\ref{fig:#1}}
\newcommand{\tab}[1]{Table~\ref{tab:#1}}

\newcommand{\sect}[1]{section~\ref{sec:#1}}

\newcommand{\thm}[1]{Theorem~\ref{thm:#1}}

\newcommand{\pro}[1]{Proposition~\ref{pro:#1}}

\newcommand{\defn}[1]{definition~\ref{def:#1}}

\newcommand{\lem}[1]{lemma~\ref{lem:#1}}

\newcommand{\SKIP}[1]{} 

\newcommand{\mbegin} {\left [ \begin{array}}

\newcommand{\mend}   {\end{array} \right ]}

\newcommand{\detbegin} {\left | \begin{array}}
\newcommand{\detend}   {\end{array} \right |}

\newcommand{\vbegin} {\left ( \begin{array}{c}}

\newcommand{\vend} {\end{array}\right )}

\def\squareforqed{\hbox{\rlap{$\sqcap$}$\sqcup$}}

\def\qed{\ifmmode\squareforqed\else{\unskip\nobreak\hfil
	\penalty50\hskip1em\null\nobreak\hfil\squareforqed
	\parfillskip=0pt\finalhyphendemerits=0\endgraf}\fi}

\def\vec#1{\mathchoice%
	{\mbox{\bf $\displaystyle\bf#1$}}
	{\mbox{\bf $\textstyle\bf#1$}}
	{\mbox{\bf $\scriptstyle\bf#1$}}
	{\mbox{\bf $\scriptscriptstyle\bf#1$}}}
\def\v#1{\protect\vec #1}

\newcommand{\showeqnlabel}{
	\hbox to 0pt{\quad\quad\relax\fbox{\scriptsize\rm\eqnlblx}%
	\gdef\eqnlblx{xxxx}}} \newcommand{\eqnlblx}{}

\def\@eqnnum{\rm (\theequation)\showeqnlabel}

\newcommand{\nofig}[1]{\centerline{\bf Figure here}}

\def\mat#1{\mathchoice{\mbox{\bf$\displaystyle\tt#1$}}
	{\mbox{\bf$\textstyle\tt#1$}}
	{\mbox{\bf$\scriptstyle\tt#1$}}
	{\mbox{\bf$\scriptscriptstyle\tt#1$}}}
\def\m#1{\protect\mat #1}

\newif\ifsupp
\suppfalse

\newif\ifarxiv
\arxivfalse

\newif\iffinal
\finalfalse

\usepackage[page]{appendix}

\newcommand{\calD}{{\cal D}}

\newcommand{\calY}{{\cal Y}}
\newcommand{\dom}{{\cal D}}
\newcommand{\domx}{{{\cal D}_{\!X}}}
\newcommand{\domz}{{{\cal D}_{\!Z}}}

\newcommand{\haty}{{\hat y}}

\newcommand{\argmin}{\operatornamewithlimits{{\rm argmin}}}
\newcommand{\argmax}{\operatornamewithlimits{{\rm argmax}}}

\input{theorenv}
\definecolor{orange}{rgb}{1,0.5,0}

\title{Post-hoc Calibration of Neural Networks by $g$-Layers}
\author{
    Amir Rahimi $^\dagger$\textsuperscript{\rm1}, Thomas Mensink $^\dagger$\textsuperscript{\rm2}, Kartik Gupta \textsuperscript{\rm1,3}, Thalaiyasingam Ajanthan \textsuperscript{\rm1}, \\ Cristian Sminchisescu \textsuperscript{\rm2}, Richard Hartley \textsuperscript{\rm1,2}
}
\affiliations{
    \textsuperscript{\rm 1}Australian National University
    \textsuperscript{\rm 2}Google Research 
    \textsuperscript{\rm 3}DATA61, CSIRO \\
    $^\dagger$ Equal Contribution

}

\begin{document}
\maketitle
\begin{abstract}
Calibration of neural networks is a critical aspect to consider when incorporating machine
learning models in real-world decision-making systems where the confidence of decisions are
equally important as the decisions themselves. 
In recent years, there is a surge of research on neural network calibration and the
majority of the works can be categorized into {\em post-hoc} calibration methods, defined
as methods that learn an additional function to calibrate an already trained base network.
In this work, we intend to understand the post-hoc calibration methods from a theoretical
point of view.
Especially, it is known that minimizing Negative Log-Likelihood (NLL) will lead to a
calibrated network on the training set if the global optimum is attained
~\cite{bishop1994mixture}.
Nevertheless, it is not clear learning an additional function in a post-hoc manner would
lead to calibration in the theoretical sense.
To this end, we prove that even though the base network ($f$) does not lead to the global
optimum of NLL, by adding additional $g$-layers and minimizing NLL by optimizing the
parameters of $g$ one can obtain a calibrated network $g \circ f$.
This not only provides a  less stringent condition to obtain a calibrated 
network but also provides
a theoretical justification of post-hoc calibration methods.
Our experiments on various image classification benchmarks confirm the theory.
    
\end{abstract}

\section{Introduction}
In this paper we consider the problem of calibration of neural networks,
or classification functions in general.  This problem has been considered
in the context of Support Vector Machines~\cite{platt1999probabilistic}, 
but has recently been 
considered in the context of Convolutional Neural Networks
(CNNs)~\cite{guo2017calibration}. 
In this case, a CNN used for classification
takes an input $x \in \domx$, belonging to one of $n$ classes,
and outputs a vector $f(x)$ in $\R^n$, where
the $y$-th component, $f_y(x)$ is often interpreted as a probability
that input $x$ belongs to class $y$.  If this value is to represent probabilities
accurately, then we require that $f_y(x) = P(y ~|~ f(x))$.
In this case, the classifier $f$ is said to be {\em calibrated},
or {\em multi-class calibrated}.
\footnote{
In many papers, \eg~\cite{kull2019beyond} and calibration metrics, \eg ECE~\cite{naeini2015obtaining}
a slightly different condition known as {\em classwise calibration} 
is preferred: ${f_y(x) = P(y ~|~ f_y(x))}$.
}

\begin{figure}[t]
    \centering
   \includegraphics[width=0.8\linewidth]{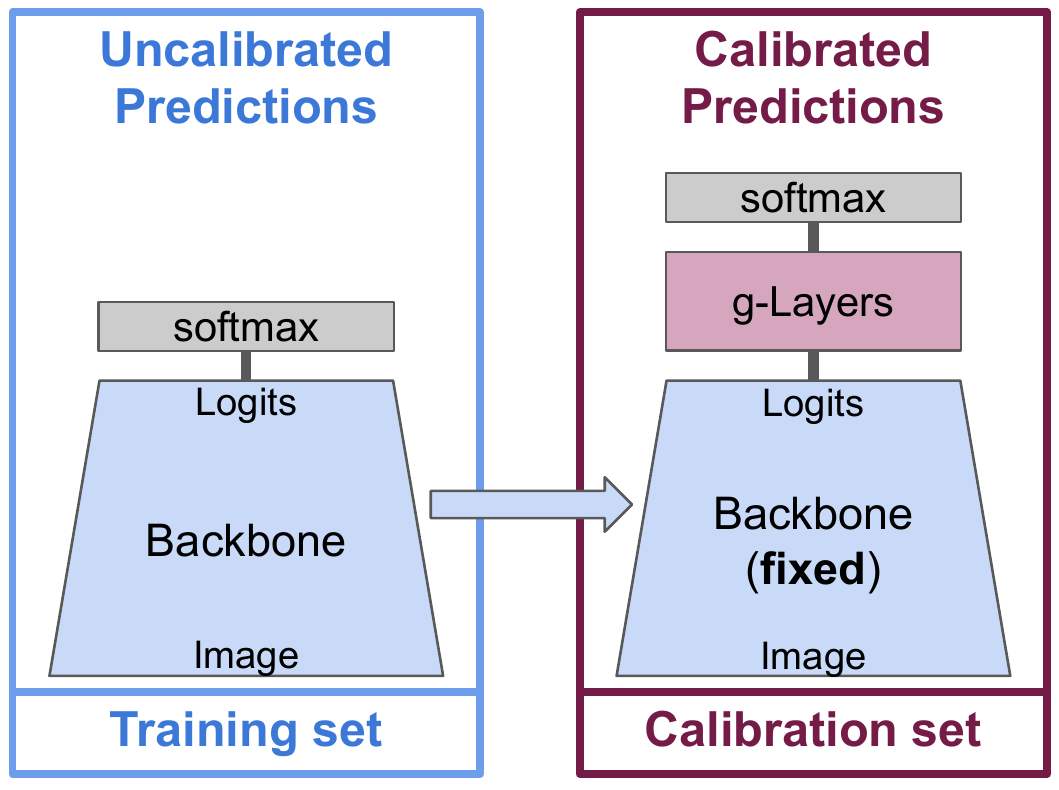}
   \caption{Illustration of $g$-layers: 
   (\emph{left}) a network $f$ is trained as usual;
   (\emph{right}) $g$-layers are trained on top of the (fixed) $f$ using an unseen calibration set. 
   The resulting network $g(f(x))$ is both \emph{theoretically} as well as \emph{empirical} calibrated.
   }
    \label{fig:network}
\end{figure}

\begin{figure*}
    \begin{subfigure}{0.23\textwidth}
        \includegraphics[width=\textwidth]{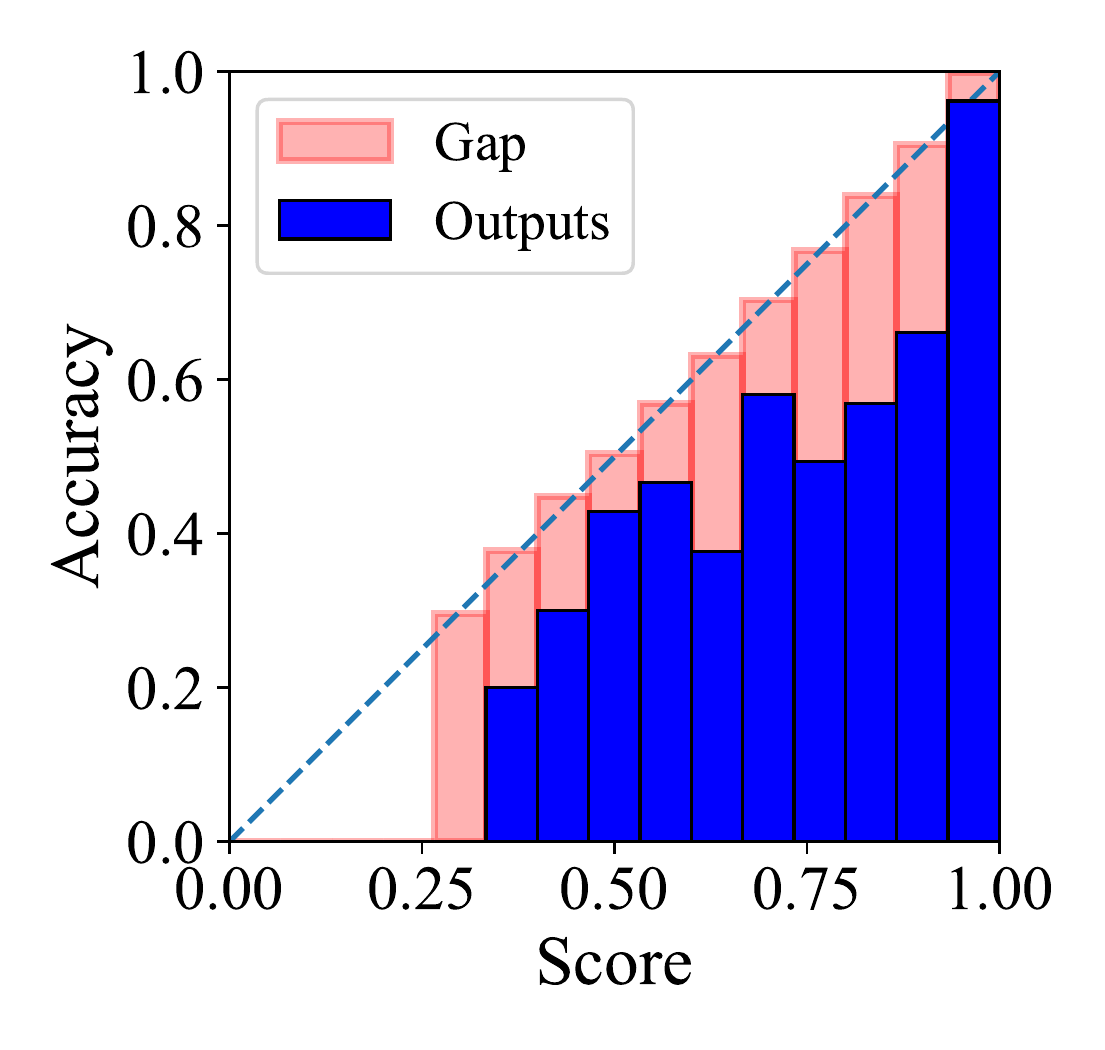}
        \vspace{-7mm}
        \caption{}
    \end{subfigure}    
    \begin{subfigure}{0.23\textwidth}
        \includegraphics[width=\textwidth]{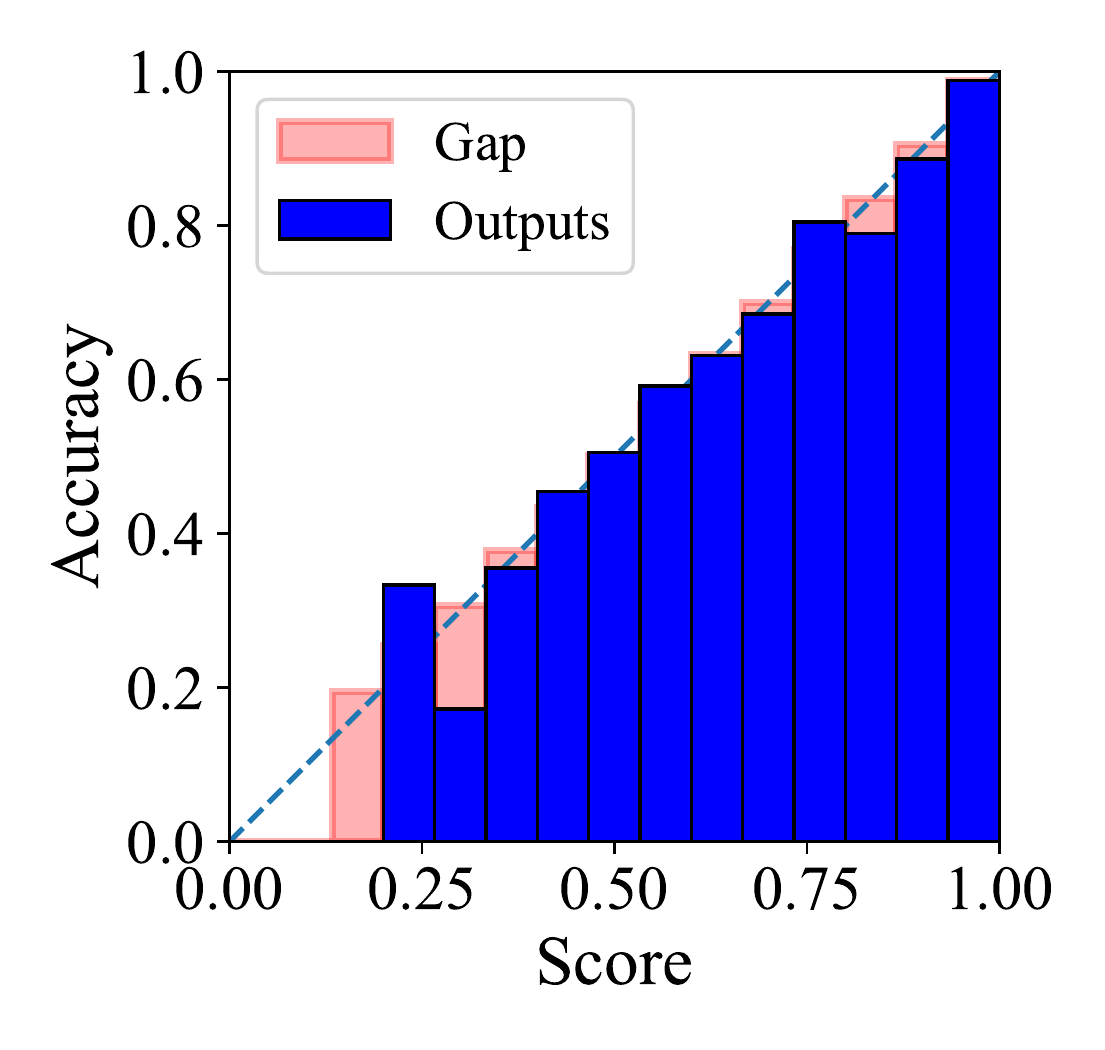}
        \vspace{-7mm}
        \caption{}
    \end{subfigure}        
    \begin{subfigure}{0.25\textwidth}
    
        \includegraphics[width=\textwidth]{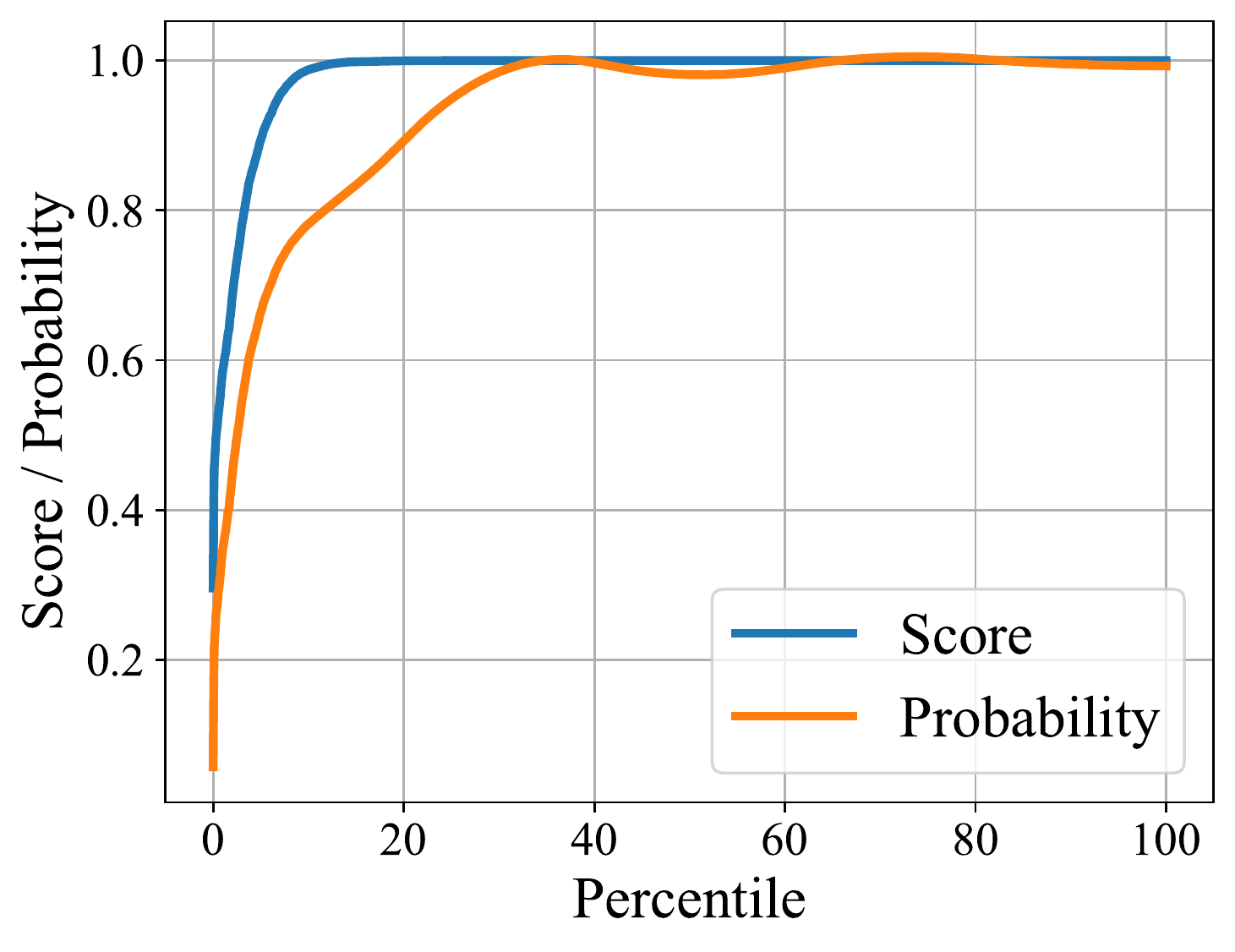}
        \caption{}
    \end{subfigure}
    \begin{subfigure}{0.25\textwidth}
        \includegraphics[width=\textwidth]{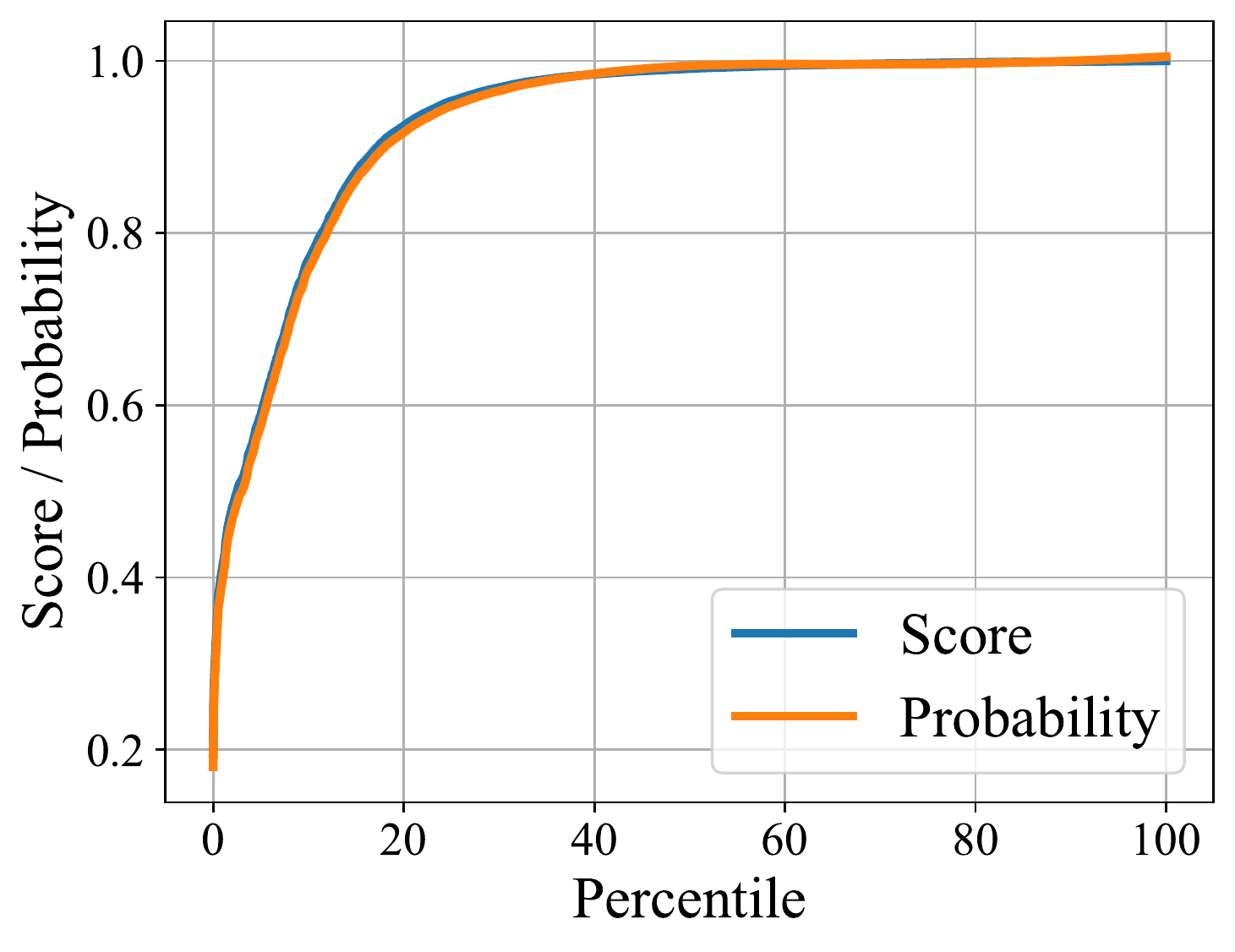}
        \caption{}
    \end{subfigure}    
\vspace{-1ex}
\caption{
Illustration of ECE and KS calibration metrics for ResNet-110 \cite{huang2017densely} model trained on CIFAR-10 for the \emph{top class}.
The uncalibrated network obtains a KS error of 4.8\%, and top-$1$ accuracy of 93.6\% on the unseen test set, the calibrated network (with two $g$-layers with 32 units each) has an KS error of 0.9\%. Reliability diagrams are shown in \textbf{a} and \textbf{b}, while \textbf{c}) and (\textbf{d} show score and probability plotted against fractile for uncalibrated and calibrated networks. If the network is perfectly calibrated, the scores and probability plot will coincide with each other as can be observed in (\textbf{d}). 
}
\label{fig:resnet110_RD_KS}
\end{figure*}


A well-known condition (\cite{bishop1994mixture}) for a classifier to be 
calibrated is that it minimizes the {\em cross-entropy} cost function, over all functions
$f: \domx \rightarrow \Delta^{n-1}$, where $\Delta^{n-1}$ is the standard
probability simplex. 
If the absolute minimum is attained, it is true that
$f_y(x) = P(y ~|~ x)$.  However, this condition is rarely satisfied, since
$\domx$ may be a very large space (for instance a set of images, of very high
dimension) and the task of finding the absolute (or even a local) minimum of 
the loss is difficult: 
it requires the network to have
sufficient capacity, and also that the network manages to find the optimal value
through training. To fulfil this requirement, two networks that reach different
minima of the loss function cannot both be calibrated. 
However, the requirement that a network is calibrated could be separated from that of finding the optimal classifier.  

In this paper, it is shown that a far less stringent condition is sufficient for
the network to be calibrated; we say that the network is {\em optimal with respect to calibration} provided no adjustment of the output of the network {\em in the output space}
can improve the calibration (see \defn{calibration-optimal}). This 
is a far simpler problem, since it requires that a function between far smaller-dimensional
spaces should be optimal.  

We achieve optimality with respect to calibration by addition of extra layers at the end of the network and post-hoc training on a hold-out calibration set to minimize the \emph{cross-entropy} cost function, see \fig{network}.
The extra layers (which we call $g$-layers) take as input the logits (\ie the output before applying the softmax of the original network) and outputs probabilities (\ie with a softmax as final activation).
Since the output space of the network is of small dimension (compared to the input of the whole network), optimization of the loss by training the $g$-layers is a far easier task.


We conduct experiments on various image classification datasets by learning a small
fully-connected network for the $g$-layers on a hold-out calibration set and evaluate on an unseen test set.
Our experiments confirm the theory that if the calibration set and the test set are
statistically similar, our method outperforms existing post-hoc calibration methods while retaining the original accuracy.

\section{Preliminaries}
We consider a pair of joint random variables, $(X, Y)$.
Random variable $X$ should take values in some domain $\domx$
for instance a set of images, and
$Y$ takes values in a finite set of classes 
$\calY = \{1, 2, \ldots, n\}$.
The variable $n$ will refer always to the number of labels, and $y$ denotes an element of the class set.

We shall be concerned with a (measurable) function $f:\domx \rightarrow \dom_Z = \R^m$,
and random variable $Z$ defined by $Z = f(X) = (f_1(X), f_2(X), \ldots, f_m(X))$. Note that $\domz$ is the same
as $\R^m$, but we shall usually use the notation $\domz$ to remind us that
it is the range of function $f$.
 The distribution of the random
variable $X$ induces the distribution for the random variable $Z = f(X)$.
The symbol $z$ will always represent $f(x)$ where $x$ is a value
of random variable $X$.
The notation $x \sim X$ means that $x$ is a value sampled from
the random variable $X$.  The situation we have in mind is that $f$ is the function
implemented by a (convolutional) neural network.
A notation $P(\cdot)$ (with upper-case $P$) always refers to probability,
whereas a lower case $p$ represents a probability distribution.
We use the notation $P(y ~|~ z)$ for brevity to mean
$P(Y=y ~|~ Z=z)$.

A common way of doing classification, given $n$ classes,
is that the neural net is terminated with a layer represented
by a function $q: \R^m \rightarrow \R^n$ (where typically $m = n$, but this is not required),
taking value $q(z)  = (q_1(z), \ldots, q_n(z))$ in $\R^n$, and satisfying 
$q_i(z) >   0$ and
$\sum_{i=1}^n q_i(z) = 1$.
The set of such vectors $q(z)$ satisfying these conditions
is called the {\em standard probability simplex}, $\Delta^{n-1}$,
or simply the standard (open) simplex.  This is an $n-1$ dimensional
subset of $\R^n$.
An example of such a function $q$ is the softmax function defined by
$
q_i(z) = \exp(z_i)/\sum_{j=1}^n \exp(z_j)
$.
%

Thus, the function implemented by a neural net is $q \circ f$, 
where $f:\domx \rightarrow \domz = \R^m$,
and ${q: \R^m \rightarrow \Delta^{n-1}}$.  The function $q$ will be
called the {\em activation} in this paper.
A function such as $q\circ f: \domx \rightarrow \Delta^{n-1}$ will
be called a {\em network}.
The notation $q\circ f$ represents the composition of the two functions
$f$ and $q$.
One is tempted to declare (or hope) that $q_y(z) = P(y ~|~ z)$, 
in other words
that the neural network outputs the correct conditional 
class probabilities given the network output.  
At least it is assumed that the most probable 
class assignment is equal to $\argmax_{y\in\calY} q_y(z)$.
It will be investigated how
justified these assumptions are.  Clearly, since $f$ can be
any function, this is not going to be true in general.

\paragraph{Loss.}
%
%
%
When using the negative log-likelihood (or cross-entropy) loss,
the expected loss over the distribution given by
the random variables $(X, Y)$ is
\begin{align}
L(q\circ f, X, Y) &= E_{(x, y)\sim(X,Y)} L(q\circ f, x, y) \nonumber\\
&= -E_{(x, y)\sim(X,Y)} \log(q_y(f(x)))
~.
\label{eq:LfXK}
\end{align}
%


We cannot know the complete distribution
of the random variables $(X, Y)$ in a real situation, however, if the distributions are represented
by data pairs 
$
\calD = \{(x_i, y_i)_{ i=1}^N\} 
$
sampled from the distribution
of $(X, Y)$, then the expected loss is approximated by
the empirical loss
\begin{align}
L(q\circ f, \calD) &\approx -E_{(x, y)\sim\calD} \log(q_y(f(x))\nonumber\\
            &= -\sum_{i=1}^N \log(q_{y_i}(f(x_i)) ~.\label{eq:nll}
\end{align}
%

The training process of the neural network is intended to find the function  $f^*$ that minimizes
the loss in \eq{nll}, given a particular network architecture.  Thus
$ 
{f^* = \argmin_{f:\domx\rightarrow \R^m} L(q \circ f, \calD)} ~.
$ 
%


\section{Calibration}\label{sec:calib}

%

According to theory (see \cite{bishop1994mixture}), 
if a network is trained to minimize the negative log-likelihood over all possible functions, \ie:
\begin{equation}
    f^* = \argmin_{f:\domx\rightarrow \Delta^{n-1}}  -E_{(x, y)\sim(X,Y)} \log(f_y(x))
\end{equation}
then the network (function $f^*$) is calibrated, in the sense that
$f_y^*(x)= P(y ~|~ x)$, as stated in the following theorem.

\begin{theorem}
\label{thm:calibration-theorem-0}
Consider joint random variables $(X, Y)$, taking values
in $\domx$ and $\calY$ respectively, where $\domx$ is some Cartesian space.
Let $f: \domx \rightarrow \Delta^{n-1}$ be a function.
Define the loss
$
L(f,  X, Y) = -E_{(x, y)\sim(X, Y)} \log \big(f_y(x)\big)
$.
If 
$
f = \argmin_{\hat f:\domx \rightarrow \Delta^{n-1}} L(\hat f, X, Y)
$
then
$
P(y ~|~ x) = f_y (x) ~.
$
\end{theorem}

This theorem is a fundamental result, but it
leaves the following difficulties.
Even if the network is trained to completion, or trained with early-stopping,
there is no expectation that the loss will be exactly minimized
over all possible functions $f:\domx\rightarrow \Delta^{n-1}$. If this were
always the case, then research into different network architectures would
be largely superfluous.  

\SKIP{

\thm{calibration-theorem-0} is a simple corollary of the following slight generalization, stated in terms of function $q$ rather than $f$ for convenience later. 

\begin{theorem}
\label{thm:calibration-theorem-1}
Consider joint random variables $(Z, Y)$, taking values
in $\R^m$ and $\calY$ respectively.
Let $q: \R^m \rightarrow \Delta^{n-1}$ be a submersion%
\footnote{
A function between two differential manifolds $q: \mathcal{M} \rightarrow \mathcal{N}$ is called
a submersion if its differential map at point $z\in \mathcal{M}$, namely
$
Dq_z : T_z \mathcal{M}\rightarrow T_{q(z)} \mathcal{N} 
$ 
 has rank equal to the dimension of $\mathcal{N}$.
 In other words, the function is rank preserving.
 \textbf{Note}: The submersion requirement is a constraint over the simplex requirement used in the preliminaries. 
 However the softmax function remains a valid option, since it is a submersion.
 }
.
Define the loss
$
L(q,  Z, Y) = -E_{(z, y)\sim(Z, Y)} \log \big(q_y(z)\big)
$.
If 
\begin{equation}
    \text{\rm id} = \argmin_{g:\R^m \rightarrow \R^m} L(q\circ g, Z, Y),
\end{equation}
where $\text{\rm id}: \R^m\rightarrow \R^m$ is the identity function.
Then
$
P(y ~|~ z) = q_y(z)$.
\end{theorem}

The theorem is proved in the supplementary material. 
This theorem weakens the condition for optimality 
of $f$ (or $q$) by the weaker requirement that it be optimal
with respect to the prepended modification by $g: \R^m \rightarrow \R^m$.
The theorem in this form will be used to derive our main theorem.

We use the insight of this generalisation to define a less
ambitious minimum which does not require the whole network to be optimal,
as follows.

\begin{definition}
\label{def:calibration-optimal}
A function $f: \domx\rightarrow \R^m$ is said to be 
{\em optimal with respect to
calibration} for a loss-function $L(\cdot, X, Y)$ and activation $q$
if 
\begin{equation}
    \argmin_{g:\R^m\rightarrow\R^m} L(q\circ g \circ f, X, Y) = {\rm id}~,
\end{equation}
where $\text{id}: \R^m\rightarrow \R^m$ is the identity function.
\end{definition}

In other words $f$ is optimal with respect to calibration if 
replacing $f$ by the \emph{composite} function $g\circ f$
does not result in a decrease of the negative log-likelihood loss. 
 
Now, we can state and prove our theorem on calibration.
\begin{theorem}
\label{thm:calibration-theorem-main}
Consider joint random variables $(X, Y)$, taking values
in $\domx$ and $\calY$ respectively.
Let $f:\domx \rightarrow \R^m$,
and let $q: \R^m \rightarrow \Delta^{n-1}$ be a submersion.
Define the loss
\begin{equation}
L(q\circ f,  X, Y) = -E_{(x, y)\sim(X,Y)} \log\big(q_y(f(x))\big)~.
\end{equation}
If $f$ is optimal with
respect to calibration, for this cost function,
then $P(y ~|~ f(x)) =  q_y(f(x))$.
\end{theorem}

\begin{proof}
Let $Z = f(X)$, and for any $x$ define $f(x) = z$.  Then, 
\begin{align}
L(q\circ f,  X, Y) &= -E_{(x, y)\sim(X,Y)} \log\big(q_y(f(x))\big)  \nonumber\\ 
&= -E_{(z, y)\sim(Z,Y)} \log\big(q_y(z)\big)  \nonumber\\
&= L(q, Z, Y).
\end{align}
According to \thm{calibration-theorem-1}, if replacing
$q$ by $q\circ g$ will not decrease the value of the loss function
(which is the definition that $q$ is optimal with respect to recalibration),
we may conclude that 
$
P(y ~|~ z) = q_y (z) 
$
as required.
\qed
\end{proof}

\thm{calibration-theorem-main} is given for the negative log-likelihood loss, however this theorem also holds for least-squares error as indicated in~\cite{bishop1994mixture} and 
various other cost functions generally known as proper losses~\cite{buja2005loss,reid2010composite}.

\paragraph{Calibrating partially trained networks. }
According to \thm{calibration-theorem-main}, there is no need for 
the classifier network $f_\theta$ to be optimized in order for it to be calibrated.
It is sufficient that the last layer of the network (before 
the {\em softmax} layer, represented by $q$) should be optimal.  Thus, it
is possible for the classifier to be calibrated even after early-stopping
or incomplete training.

A simple rewriting of \thm{calibration-theorem-main} is

} 

We show in this paper, however that this is not necessary -- a far
weaker condition is sufficient to ensure calibration.
Instead of the loss function being optimized
over all functions $f: \domx \rightarrow \Delta^{n-1}$,
it is sufficient that the optimization be carried out over functions
$g: \R^m \rightarrow \R^m$ placed just before the activation function $q$.  Since the dimension of $\domx$
is usually very much greater that the number of classes $m$, optimizing
over all functions $g:\R^m \rightarrow \R^m$ is a far simpler task.

\begin{theorem}
\label{thm:optimum-g}
Consider joint random variables $(X, Y)$, taking values
in $\domx$ and $\calY$ respectively.
Let $f:\domx \rightarrow \R^m$, and $Z = f(X)$.
Further, let $q: \R^m \rightarrow \Delta^{n-1}$ be a submersion.
If
\begin{align}
\label{eq:optimum-g}
         g &= \argmin_{\hat g:\R^m\rightarrow \R^m} -E_{(z, y)\sim(Z, Y)} \log \big(q_y(\hat g(z))\big)~,
\end{align}
then $P(y ~|~ g\circ f(x)) =  q_y(g\circ f(x))$.
\end{theorem}

The condition that $q$ is a submersion implies (by definition) that the differential
map $dq : T\R^3 \rightarrow T\Delta^{n-1}$ is a subjection.
This required condition of the activation function $q$ being a submersion is satisfied by most activation functions, including the standard softmax activation.

In broad overview, \thm{optimum-g} is proved by applying 
\thm{calibration-theorem-0} to the function $q \circ g$ applied
to $u = f(x)$
to show that if $q \circ g$ is optimal
with respect to the loss function, then 
$P(y ~|~ q \circ g(u)) = q_y (g(u))$, which is the
same as $P(y ~|~ q\circ g \circ f(x)) = q_y (g\circ f(x))$.
The condition that $q$ be a submersion then allows the condition
that $g \circ q$ is optimal to be ``pulled back'' to a condition
that $g$ is optimal, in the sense required.  The details and profs are provided in the
supplementary material.
This theorem leads to the following definition.

\begin{definition}
\label{def:calibration-optimal}
A network $g \circ f: \domx\rightarrow \R^m$ is said to be 
{\em optimal with respect to
calibration} for a loss-function $L(\cdot, X, Y)$ and activation $q : \R^m \rightarrow \Delta^{n-1}$
if \eq{optimum-g} is satisfied.
\end{definition}

Then \thm{optimum-g} may be paraphrased by saying that the network
$q \circ g \circ f$ (where $q$ is a submersion) is calibrated if $g \circ f$ is optimal with respect to calibration.

\thm{calibration-theorem-0} and \thm{optimum-g} are given for the negative log-likelihood loss, however this theorem also holds for least-squares error as indicated in~\cite{bishop1994mixture} and 
various other cost functions generally known as proper losses~\cite{buja2005loss,reid2010composite}.

\paragraph{Calibrating partially trained networks. }
According to \thm{optimum-g}, there is no need for 
the classifier network $f_\theta$ to be optimized in order for it to be calibrated.
It is sufficient that the last layer of the network (before 
the {\em softmax} layer, represented by $q$) should be optimal.  Thus, it
is possible for the classifier to be calibrated even after early-stopping
or incomplete training.
Our calibration strategy, presented in \sect{finding-calibrated-nn}, is based on this theorem.

\paragraph{Classwise and top-$r$ calibration. }
\thm{optimum-g} gives a condition for the network to
be calibrated in the sense called {\em multi-class calibration} 
in~\cite{kull2019beyond}.  Many other calibration methods~\cite{kumar2019verified,platt1999probabilistic,zadrozny2002transforming}
aim at {\em classwise calibration}.  It can be shown (see the proofs in the
supplementary material) that if a classifier is correctly multi-class calibrated, then it is classwise calibrated as well.
The converse does not hold.

Furthermore a multi-class calibrated network is also correctly calibrated for the top-$r$ prediction, or within-top-$r$ prediction, \ie the probability of the correct class being one of the top $r$ predictions equals to the sum of the top $r$ scores.

Details are as follows.  By definition (see also ~\cite{kull2019beyond}
a network $f$ is said to be multi-class calibrated if $P(y ~|~ z) = z_k$,
where $z = f(x)$ and $z_k = f_k(x)$, which is also our definition 
of calibration.  A network is said to be {\em classwise calibrated} if
for every $k$ there is a function $f_k: \domx \rightarrow [0,1]$
such that $P(y ~|~ z_k) = z_k$, so each class is calibrated separately.
There is no requirement that $\sum_k f_k(x) = 1$.  It is shown
in our supplementary material that multi-class calibration implies
classwise classification, in that the component
functions $f_k$ derived from the function $f:\domx \rightarrow \Delta^{n-1}$
are class-calibration functions, though the converse is not true. 
(This result is not entirely trivial, because in general
$P(y ~|~ z_k) \ne P(y ~|~ z)$.) 

One can also consider {\em top-$r$ classification}, or {\em within top-$r$
classification}, which allows one to determine the probability
that the ground-truth $y$ for a sample is the $r$-th highest-scoring
classes, or within the top $r$ highest scoring samples.  

In particular, if $z \in \Delta^{n-1}$, then we denote the
$r$-th highest component of the vector $z$ by $z^r$.  Note the use of the 
upper-index to represent the numerically $r$-th highest component,
whereas $z_k$ (lower index) is the $k$-th component of $z$.
Given random variables $Y$ and $Z = f(X)$, we can also define the 
event $e^{=r}$ to mean that the ground truth $y \sim Y$ of a sample
is equal to the numerically $r$-th top component of $z$.
Similarly, $e^{\le r}$ is defined to mean that the ground truth $y$
is among the top $r$ scoring classes.
We show (see the supplementary material) that if the network is 
multi-class calibrated, then
\begin{align}
    P(e^{=r} ~|~ z) &= z^r\\
    P(e^{\le r} ~|~ z) &= \sum_{k=1}^r z^r
\end{align}


\newcommand{\bst}[1]{\textbf{#1}}
\begin{table*}[t]
	\centering
	\resizebox{\textwidth}{!}{
	\begin{tabular}{ll|cccc|ccccc}
		\toprule
        Dataset                   & Base Network              & Uncalibrated& Temp. Scaling &   MS-ODIR &  Dir-ODIR  &         \multicolumn{5}{c}{\textbf{$g$-Layers}}\\
                                  &                           &             &             &             &             &           1 &           2 &           3 &           4 &           5 \\\midrule
		CIFAR-10                  & ResNet 110                &       4.751 &  \bst{0.917}&       0.988 &       1.076 &       0.924 &       0.990 &       0.954 &       1.116 &       1.066 \\
		                          & ResNet 110 SD             &       4.103 &       0.362 &       0.331 &       0.368 &       0.317 &       0.378 &       0.342 &       0.307 &  \bst{0.188}\\
		                          & Wide ResNet 32            &       4.476 &       0.296 &  \bst{0.284}&       0.313 &       0.320 &       0.296 &       0.351 &       0.337 &       0.420 \\
		                          & DensNet 40                &       5.493 &       0.900 &       0.897 &       0.969 &       0.911 &       1.026 &  \bst{0.669}&       1.679 &       1.377 \\\midrule
		SVHN                      & ResNet 152 SD             &       0.853 &  \bst{0.553}&       0.572 &       0.588 &       0.593 &       0.561 &       0.579 &       0.588 &       0.564 \\\midrule
		CIFAR-100                 & ResNet 110                &      18.481 &       1.489 &       2.541 &       2.335 &       1.359 &       1.618 &  \bst{0.526}&       1.254 &           - \\
		                          & ResNet 110 SD             &      15.833 &       0.748 &       2.158 &       1.901 &  \bst{0.589}&       1.165 &       0.848 &       0.875 &           - \\
		                          & Wide ResNet 32            &      18.784 &       1.130 &       2.821 &       2.000 &       0.831 &  \bst{0.757}&       1.900 &       0.857 &           - \\
		                          & DensNet 40                &      21.157 &       0.305 &       2.709 &       0.775 &       0.249 &  \bst{0.188}&       0.199 &       0.203 &           - \\\midrule
	ILSVRC'12                     & ResNet 152                &       6.544 &       0.792 &       5.355 &       4.400 &       0.776 &  \bst{0.755}&       0.849 &       0.757 &           - \\
		                          & DensNet 161               &       5.721 &  \bst{0.744}&       4.333 &       3.824 &       0.881 &       1.091 &       0.780 &       1.105 &           - \\\bottomrule
	\end{tabular}
	}
	\caption{%
	KS calibration error~\cite{gupta-spline-iclr21} (Top-1 in \%) comparisons against state-of-the-art post-hoc calibration methods on several image classification datasets, using various network architectures. We vary the number of dense $g$-layers in the range 1--5. Each hidden $g$-layer has a fixed number of units depending on the number of classes (32, 302, 3002 for 10, 100, 1000 classes respectively). The results show that $g$-layers can be trained using NLL effectively reducing the KS error.
	}
	\label{tab:KS_overview}
\end{table*}

\section{Finding a Calibrated Neural Network}
\label{sec:finding-calibrated-nn}
Based on \thm{optimum-g} we propose the following strategy to find a calibrated neural network, as also illustrated in \fig{network}.
Our strategy is to replace function $f$ by $g\circ f$, where $g$ minimizes
the loss function $L(q \circ g, Z, Y)$ in \eq{optimum-g}.  Then the function $g\circ f$ 
will be 
calibrated. 
We assume that both $f$ and $g$ are implemented by a (convolutional) neural net and proceed as follows:
\begin{enumerate}
\item Train the parameters $\theta$ of a convolutional neural network $f$ on the {\em training} set to obtain $f_\theta$.
\item Strip any softmax layer (or equivalent) from $f_\theta$.
\item Capture samples $(z, y) \sim (f_\theta(X), Y)$ from a {\em calibration} set, which should be different from the training set used to train $f_\theta$.
\item Train a neural network $g$ with parameters $\phi$ on the captured $\{z, y \}$ samples to minimize \eq{optimum-g}, providing $g_\phi$.  
\item The composite network $g_\phi \circ f_\theta$ is the \textbf{calibrated} network.
\end{enumerate}

According to \thm{optimum-g}, the output of the composite network $g \circ f$ will be 
calibrated, provided that the minimum is achieved when training $g_\phi$
and that the calibration dataset
accurately represents the distribution $Z = f(X)$.  

It is a far simpler task to train a network $g_\phi$ to minimize $L(q \circ g, Z, Y)$ than
it is to train $f_\theta$ to minimize $L(q \circ f, X, Y)$, since the dimension of the
data $Z \in \R^m$ is normally far smaller than the dimension of $\domx$.
In our experiments, we implement $g$ as a small multilayer perceptron (MLP)
consisting of up to a few dense layers, of dimension no greater than a small multiple
of $m$.  Training time for $g_\phi$ is usually less than a minute.

\paragraph{Initialization. }
Assuming that function $f_\theta$ has already been trained to minimize
the loss on the training set, when $g_\phi$ is trained we do not wish
to undo all the work that has been done by starting training $g_\phi$
from an arbitrary (random) point.  Therefore, we initialize the parameters
of $g_\phi$ so that initially it implements the identity function.
We refer to these layers as {\em transparent layers}.  
This is similar to the approach in \cite{chen2015net2net}.

An alternative could be to train $g_\phi \circ f_\theta$ on the training set first, followed
by a short period of training on the calibration set, keeping the parameters
$\theta$ fixed.  In this case, it is not necessary to initialize the $g$-layers
to be transparent. The experimental validation of this alternative, however, falls beyond the scope of the current paper.


\paragraph{Overfitting. }
We note that $f_\theta$ is often well calibrated on the {\em training} set, but usually poorly calibrated on the {\em test} set.
Similarly, when using a large $g$-layer network it is relatively easy to obtain very good calibration on the {\em calibration} set, however calibration as measured on the {\em test} set, although far better than the calibration of the original network $f_\theta$, is not always as good.  In other words, also the $g$-layers are prone to overfitting to the calibration set.

The phenomenon of overfitting to the calibration set has been observed by many authors as far back as \cite{platt1999probabilistic}.  
The lesson from this is that the set used for calibration of the $g$-layers should be relatively large.  
For the CIFAR-10 dataset, we used $45,000$ training samples and $5000$ calibration samples (standard practice in calibration literature), but a different split of the data may provide better calibration results.  
In addition, the number of parameters in the $g$-layers should be kept low to avoid the risk of overfitting, hence using more than a few layers for $g_\phi$ seems counter-productive.
In practice we (also) add weight decay as regularization.

\begin{table*}[t]
	\centering
	\resizebox{\textwidth}{!}{
	\begin{tabular}{ll|cccc|ccccc}
		\toprule
        Dataset                   & Base Network              & Uncalibrated & Temp. Scaling &  MS-ODIR &  Dir-ODIR &         \multicolumn{5}{c}{\textbf{$g$-Layers}}\\
                                  &                           &             &             &             &             &           1 &           2 &           3 &           4 &           5 \\\midrule
        CIFAR-10                  & ResNet 110                &       4.750 &       1.132 &       1.052 &       1.144 &       1.130 &  \bst{0.997}&       1.348 &       1.152 &       1.219 \\
		                          & ResNet 110 SD             &       4.113 &       0.555 &       0.599 &       0.739 &       0.807 &       0.809 &       0.629 &       0.674 &  \bst{0.503}\\
		                          & Wide ResNet 32            &       4.505 &       0.784 &       0.784 &       0.796 &  \bst{0.616}&       0.661 &       0.634 &       0.669 &       0.670 \\
		                          & DensNet 40                &       5.500 &       0.946 &       1.006 &       1.095 &       1.101 &       1.037 &  \bst{0.825}&       1.729 &       1.547 \\\midrule
		SVHN                      & ResNet 152 SD             &       0.862 &       0.607 &       0.616 &       0.590 &       0.638 &  \bst{0.565}&       0.589 &       0.604 &       0.648 \\\midrule
		CIFAR-100                 & ResNet 110                &      18.480 &       2.380 &       2.718 &       2.896 &       2.396 &       2.334 &  \bst{1.595}&       1.792 &           - \\		
		                          & ResNet 110 SD             &      15.861 &  \bst{1.214}&       2.203 &       2.047 &       1.219 &       1.405 &       1.423 &       1.298 &           - \\

		                          & Wide ResNet 32            &      18.784 &       1.472 &       2.821 &       1.991 &       1.277 &  \bst{1.096}&       2.199 &       1.743 &           - \\
		                          & DensNet 40                &      21.156 &       0.902 &       2.709 &       0.962 &       0.927 &       0.644 &       0.562 &  \bst{0.415}&           - \\\midrule
    ILSVRC'12                     & ResNet 152                &       6.543 &       2.077 &       5.353 &       4.491 &       2.025 &       2.051 &  \bst{1.994}&       2.063 &           - \\
		                          & DensNet 161               &       5.720 &       1.942 &       4.333 &       3.926 &       1.952 &       1.978 &  \bst{1.913}&       1.931 &           - \\\bottomrule
	\end{tabular}
	}
	\caption{%
	Expected calibration error (ECE)~\cite{naeini2015obtaining} (Top-1 in \%) comparisons against state-of-the-art post-hoc calibration methods on several image classification datasets using various network architectures. We vary the number of dense $g$-layers in the range 1--5. Each hidden $g$-layer has a fixed number of units depending on the number of classes (32, 302, 3002 for 10, 100, 1000 classes respectively). The results show that $g$-layers can be trained using NLL effectively reducing the ECE error.
	}
	\label{tab:ECE_overview}
\end{table*}

\section{Related Work}
Calibrating classification functions has been studied for the past few decades,
earlier in the context of support vector
machines~\cite{platt1999probabilistic,zadrozny2002transforming}  
and recently on neural networks~\cite{guo2017calibration}.
In this literature, it is typically preferred to calibrate an already trained 
classifier, denoted as {\em post-hoc calibration}, as it can be applied 
to any off-the-shelf classifier.
Over the past few years, many post-hoc calibration methods have been
developed such as temperature scaling~\cite{guo2017calibration} 
(as an adaptation of Platt scaling~\cite{platt1999probabilistic} 
for multi-class classification), Bayesian
binning~\cite{naeini2015obtaining}, beta calibration~\cite{kull2017beta} and its
extensions~\cite{kull2019beyond} to name a few.
These methods are learned on a hold-out calibration set and 
the main difference among them is the type of function learned and 
the heuristics used to avoid overfitting to the calibration set.
Specifically, temperature scaling learns a scalar parameter while vector or matrix 
scaling learns a linear transformation of the classifier
outputs~\cite{guo2017calibration}.
Later, additional regularization constraints such as penalizing 
off-diagonal terms~\cite{kull2019beyond} 
and order-preserving constraints~\cite{rahimi2020intra} 
are introduced to improve 
matrix scaling.
While several practical methods are developed in this regime, 
it was not clear previously whether learning a calibration function post-hoc would lead to calibration in the theoretical sense.
We precisely answer this question and provide a theoretical justification 
of these methods.
Even though, the proof is provided for negative-log loss, it is
 applicable to any proper
loss function~\cite{buja2005loss,reid2010composite}.

We would like to clarify that our theoretical result (similar to~\cite{bishop1994mixture})
is obtained under the assumption that the calibration set matches the true data distribution 
(or simply the test set distribution).
Similar to the assumption used to train most classifiers.
In this regard, there have been various techniques introduced, such as label smoothing~\cite{muller2019does} and data augmentation~\cite{mixup}, to avoid overfitting while training a classification (base) network. We believe those techniques are applicable in the calibration context as well.

\begin{figure*}
    \centering
    \begin{subfigure}[b]{0.31\textwidth}
        \includegraphics[height=30mm]{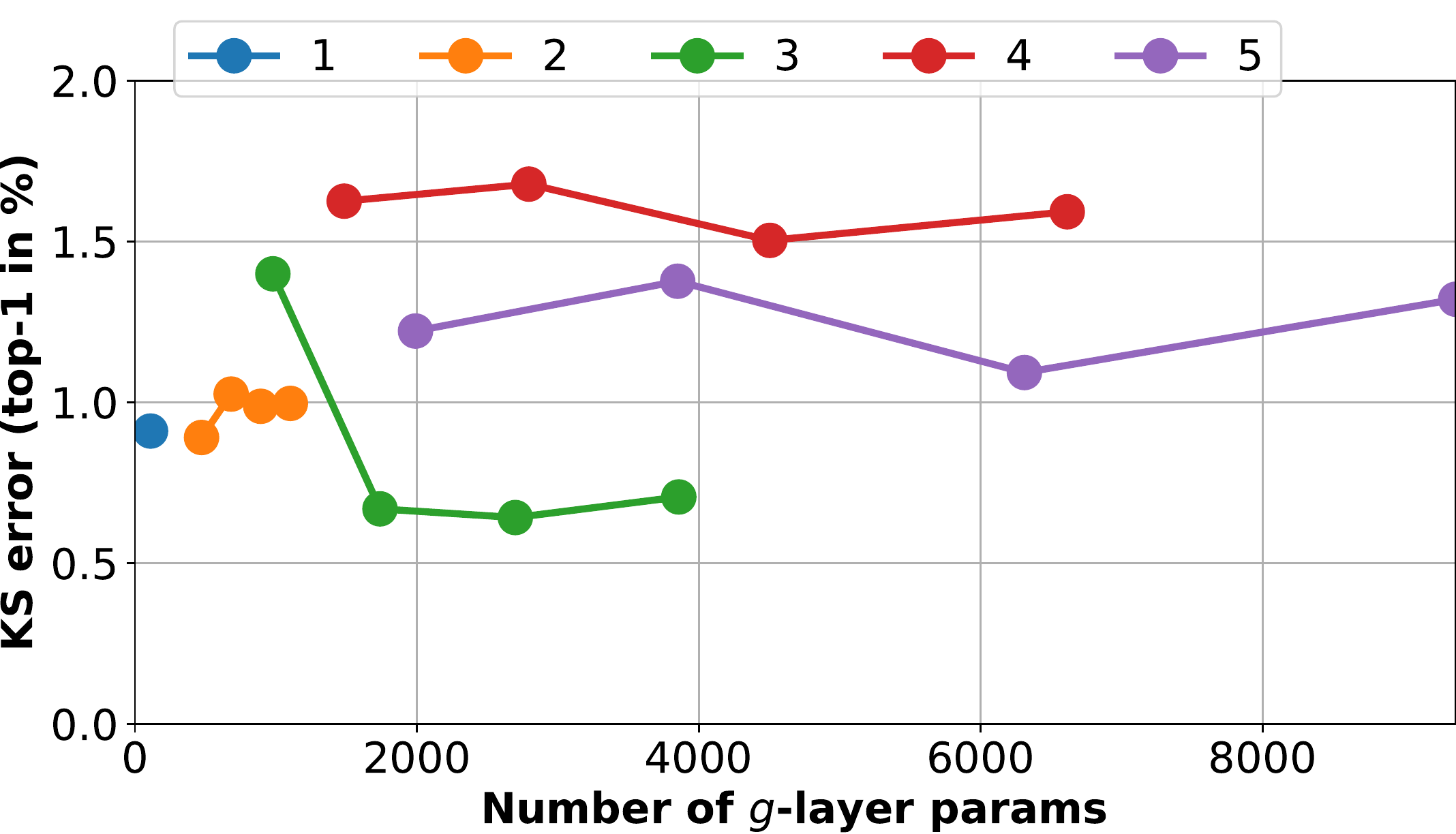}
        \caption{DenseNet 40 (CIFAR-10)}
    \end{subfigure}
    ~
    \begin{subfigure}[b]{0.31\textwidth}
        \includegraphics[height=30mm]{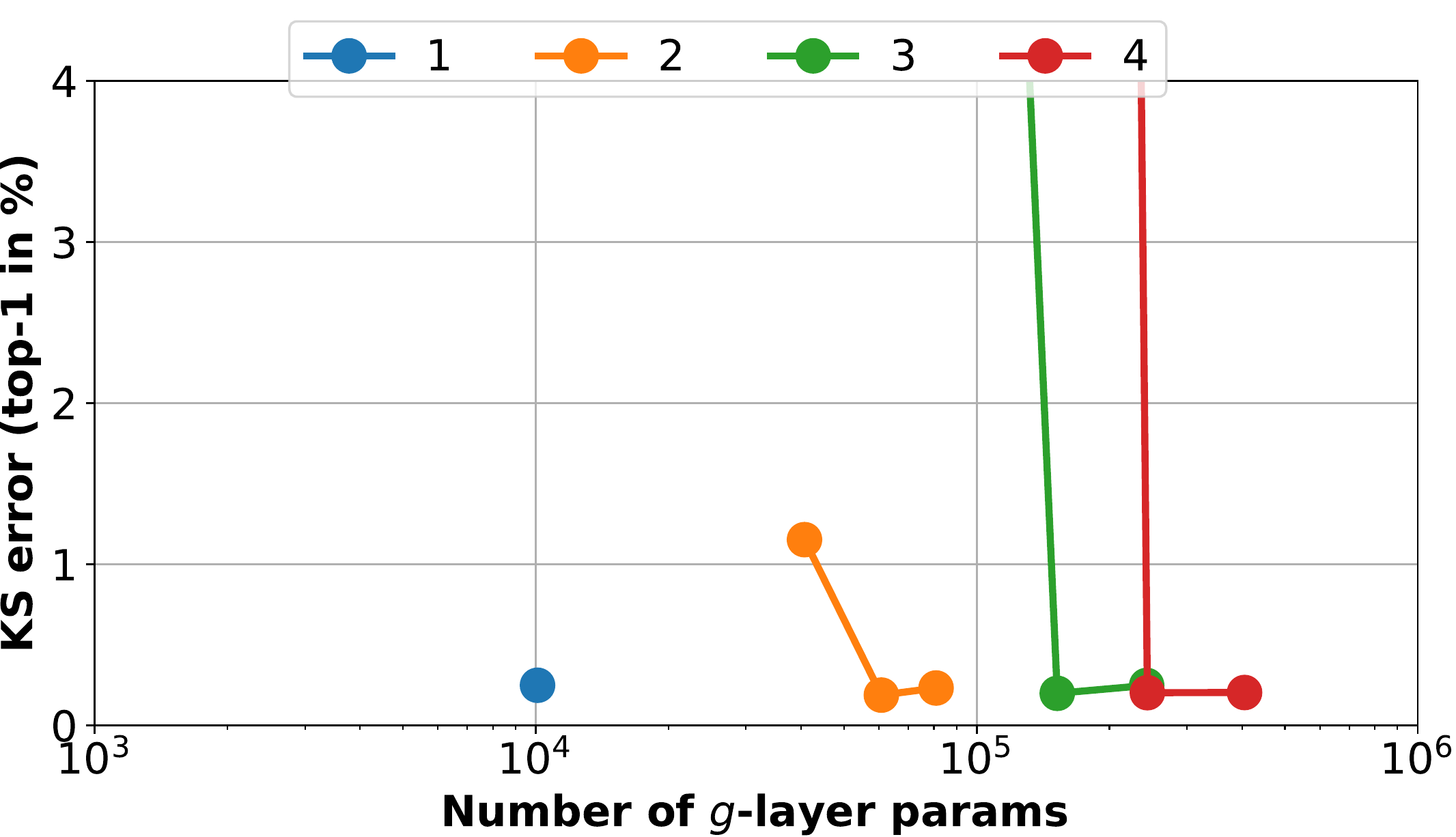}
        \caption{DenseNet 40 (CIFAR-100)}
    \end{subfigure}    
    ~
    \begin{subfigure}[b]{0.31\textwidth}
        \includegraphics[height=30mm]{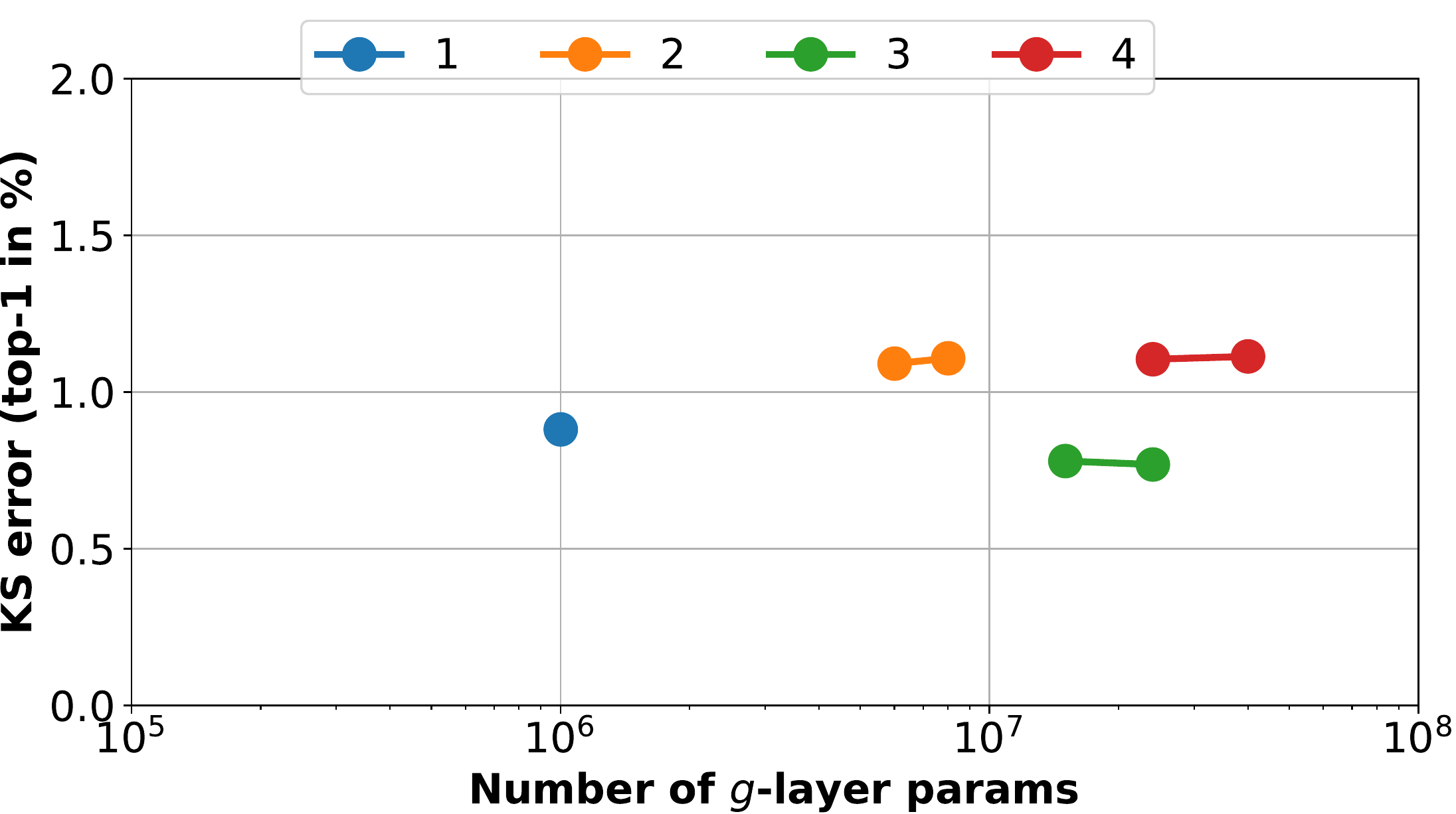}
        \caption{DenseNet 161 (ILSVRC'12)}
    \end{subfigure}
    \caption{
    Plots of KS error (top-1 in \%) as a function of the number of parameters in the $g$-layer networks. For the $g$-layers, the number of dense layers is varied (indicated by the different colors) and number of dense units per layer. Here we use $\{2, 3, 4, 5\}$ units per class for CIFAR-10, $\{2, 3, 4\}$ for CIFAR-100, and $\{3, 4\}$ for ILSVRC'12.
    }
    \label{fig:ndnu}
\end{figure*}
\newcolumntype{H}{>{\setbox0=\hbox\bgroup}c<{\egroup}@{}}
\begin{table}[t]
	\centering
	\resizebox{\columnwidth}{!}{
	\renewcommand{\bst}[1]{#1}
	\begin{tabular}{llcHHH|HHcHH}
		\toprule
                Dataset                   & Base Network              &         Uncalibrated & Temp. Scaling &          MS &         Dir &         \multicolumn{5}{c}{\textbf{$g$-Layers}}\\\midrule
		CIFAR-10                  & ResNet 110                &       1.066 &       0.275 &       0.277 &       0.283 &       0.283 &       0.285 &       0.283 &       0.291 &  \bst{0.269}\\                
		                          & ResNet 110 SD             &       0.918 &       0.118 &       0.113 &       0.114 &       0.117 &       0.113 &       0.108 &       0.105 &  \bst{0.088}\\
		                          & Wide ResNet 32            &       0.995 &       0.123 &       0.119 &       0.117 &  \bst{0.101}&       0.107 &       0.121 &       0.114 &       0.127 \\
		                          & DensNet 40                &       1.234 &       0.225 &       0.223 &       0.239 &       0.236 &       0.249 &  \bst{0.174}&       0.387 &       0.325 \\\midrule
		SVHN                      & ResNet 152 SD             &       0.188 &       0.134 &  \bst{0.133}&  \bst{0.133}&       0.142 &       0.134 &       0.138 &       0.137 &       0.134 \\\midrule
		CIFAR-100                 & ResNet 110                &       3.299 &       0.462 &       0.579 &       0.540 &       0.443 &       0.488 &  \bst{0.336}&       0.401 &           - \\		
		                          & ResNet 110 SD             &       2.904 &       0.241 &       0.434 &       0.407 &  \bst{0.210}&       0.277 &       0.234 &       0.234 &           - \\
		                          & Wide ResNet 32            &       3.459 &       0.348 &       0.574 &       0.445 &       0.304 &  \bst{0.299}&       0.462 &       0.338 &           - \\
		                          & DensNet 40                &       3.877 &       0.181 &       0.514 &       0.219 &       0.158 &  \bst{0.128}&       0.169 &       0.183 &           - \\\midrule
		ILSVRC'12                 & ResNet 152                &       1.155 &       0.329 &       0.942 &       0.839 &       0.329 &  \bst{0.326}&       0.329 &       0.333 &           - \\
		                          & DensNet 161               &       1.040 &       0.316 &       0.783 &       0.751 &       0.334 &       0.363 &  \bst{0.306}&       0.361 &           - \\\bottomrule
	\end{tabular}
	}
	\caption{
	Multi-class calibration as measured by the average Top-10 KS metric.
	Across all networks and datasets $g$-layers significantly improve the multi-class calibration.
	}
	\label{tab:ECE_top10}
\end{table}

\section{Experiments}

\subsection{Experimental setup}
For our experimental validation we calibrate deep convolutional neural networks trained on the CIFAR-10/CIFAR-100~\cite{krizhevsky2009learning}, SVHN~\cite{netzer2011reading} and ILSVRC'12~\cite{ILSVRC15} datasets. 
For the base network $f$ we use pre-trained models of different architectures: ResNet~\cite{he2016deep}, ResNet Stochastic Depth~\cite{huang2016eccv}, DenseNet~\cite{huang2017densely}, and Wide ResNet~\cite{zagoruyko2016wide}.
For most of the experiments we use the pre-trained models also used in~\cite{kull2019beyond}.

\paragraph{Initialization.}
The proposed $g$-layers are initialized with transparent layers so that at initialization they represent an identity mapping similar to the idea in~\cite{chen2015net2net}.
Preliminary results have shown that this transparent initialisation is necessary to train the $g$-layers from the relatively small calibration set. Especially for datasets with many classes the accuracy will drop significantly when normal random initialised weights are used.

\paragraph{Training.}
The $g$-layers are trained on a {\em calibration} set (not used for training base networks nor for evaluation). 
The hyper-parameters for learning (learning rate and weight decay) are determined using 5-fold cross validation on the calibration set.
The best results are used to train the $g$-layers on the full calibration set.
For all models (cross-validation and final calibration) early stopping is used based on the negative log-likelihood of the current training set.

\paragraph{Evaluation.}
To evaluate the $g$-layers, the calibration error is evaluated on the {\em test} set.
While our theory as well as our approach guarantees multi-class calibration (\cf \sect{calib}), in the calibration literature~\cite{guo2017calibration,kull2019beyond}, the standard practice is to measure the calibration of top-$1$ predictions (or generally classwise calibration).

We measure top-1 calibration error using Expected Calibration Error (ECE)~\cite{naeini2015obtaining} and Kolmogorov-Smirnov calibration error (KS-error)~\cite{gupta-spline-iclr21}.
While ECE is a widely used metric, a known weakness is its dependence on histograms (see \fig{resnet110_RD_KS}~(a) and (b)) which is deemed as a weakness since the final error depends on the chosen histogram binning scheme.
ECE might be particularly unsuitable on deep networks trained on small datasets such as CIFAR-10, since over $90$\% of scores are over $0.9$, and hence lie in a single bin (see \fig{resnet110_RD_KS} (c), which plots the scores versus fractile).

The KS-error~\cite{gupta-spline-iclr21} computes the maximum difference between the cumulative predicted distribution and the true cumulative distribution. 
If the network is consistently over- or under-calibrated, which is usually the case,
then the KS-error measures the (empirical) expected absolute 
difference between the score $q_i(g(f(x)))$ and the probability $P(i ~|~ q_i(g(f(x))))$, where $i = \argmax_{y \in \calY} q_y(g(f(x)))$. 
The cumulative distributions also provides visualizations similar to reliability diagrams, see \eg \fig{top_ks_error}.

\begin{figure*}
    \centering
    \begin{subfigure}[b]{\textwidth}
        \includegraphics[width=\textwidth,trim=0px 0px 450px 0px, clip]{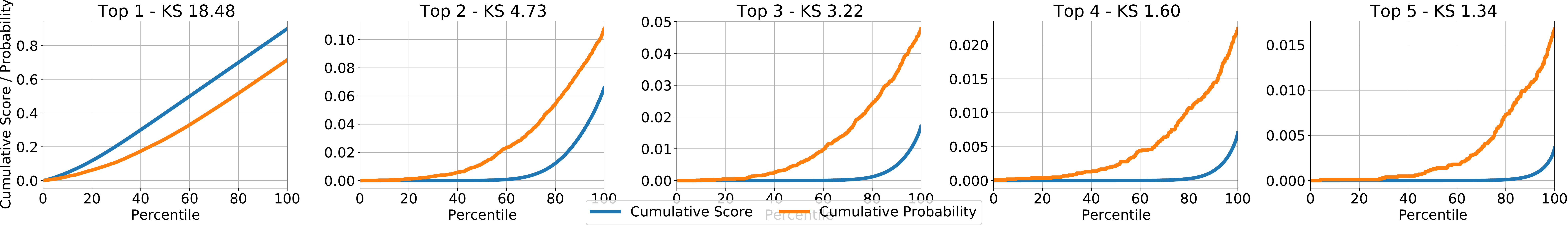}
        \caption{Uncalibrated}
    \end{subfigure}
    \begin{subfigure}[b]{\textwidth}
        \includegraphics[width=\textwidth,trim=0px 0px 450px 0px, clip]{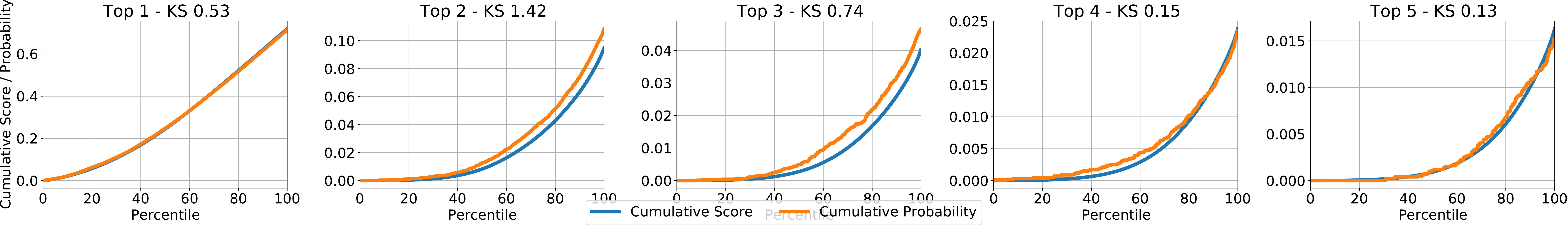}
        \caption{Calibrated $g$-layer (with 3 layers and 32 hidden units)}
    \end{subfigure}
    \caption{
    Comparison of the KS error of the Top-1 to Top-4 classes between the uncalibrated ResNet 110 model (\emph{top}) and the calibrated $g$-layer variant (\emph{bottom}), on CIFAR-100. The $g$-layer model significantly reduces the KS error for all Top-$k$ classes.
    }
    \label{fig:top_ks_error}
\end{figure*}
 
\subsection{Results}
We first provide an experimental comparison with other post-hoc calibration methods on the dataset used.
Then, we discuss in more depth the between the number of parameters in $g$-layers and overfitting with respect to calibration for a given dataset. 
In short, as predicted by our theory, if overfitting to the calibration set is reduced in practice, learning complete $g$-layers lead to superior calibration.
Nevertheless, heuristics for mitigating overfitting such as using larger calibration set and increased regularization (dropout, weight-decay, etc.) are relevant and the best approach to avoid overfitting with respect to calibration remains an open question.
 
\paragraph{Comparisons to other methods.}
In this set of experiments we compare $g$-layer calibration to several other calibration methods, including temperature scaling~\cite{guo2017calibration}, 
MS-ODIR~\cite{kull2019beyond}, and Dir-ODIR~\cite{kull2019beyond}. 
To provide the calibration results using the baseline methods, we use the base models and implementation of~\cite{kull2019beyond}. 
Since, we do not need to retrain the base models, we train $g$-layers on top of the pre-trained models.

We train $g$-layers with different number of dense layers, in the range from 1--5.
The size of the hidden $g$-layers is fixed to $H = 3 \times C + 2$, where C is the number of classes in the dataset. This satisfies the requirement for the transparent initialisation that $H > 2C$. In practice this means that the number of weights in $g$-layers scale cubic with the number of classes, \eg the 3-layer network for ILSVRC'12 contains 15M weights, while the 4-layer network has 24M.

The performance is measured using KS-Error, in~\tab{KS_overview}, and ECE, in~\tab{ECE_overview}. 
We observe that the proposed $g$-layers achieve at least comparable calibration performance to the current state-of-the-art methods, but often (significant) better. 
In general there is a negligible effect on the accuracy ($\pm0.5\%$) of the base network, see supplementary material for full results.

We would like to point out that all the compared methods belong to the post-hoc calibration category and can be thought of as special cases of our method (that is learning a $g$-function).
The main difference between these methods is the allowed function class while optimizing the $g$-layers, which can be thought of as a technique to avoid overfitting on a small calibration set.

For our \emph{dense} $g$-layers holds that the 5-fold cross validation seems to be able to find good learning hyper-parameters, mitigating overfitting when training (large) networks on a relatively small calibration set.
We conclude that we can train $g$-layers effectively for different network architectures, using a range of hidden layers for various datasets. 


\paragraph{Number of hidden units}
In this set of experiments we explore the number of hidden units used in the $g$-layer network and their relation to the calibration performance. 
For these experiments we train $g$-layers on the DenseNet models from CIFAR-10/100 and ILSVRC'12 (other networks/datasets are provided in the supplementary). 
The number of dense layers is varied, in the range 1 -- 5, and the number of hidden units is set to $H=h\times C + 2$.

The results are in \fig{ndnu}.
From these results we observe that in most cases the performance is stable across the number of hidden units and the number of layers.
This shows that dense $g$-layers can be trained effectively over large number of parameters, when initialised with transparent layers and learning settings found by cross validation. From these results we do not see clear signs of overfitting. 

The bad performance on the CIFAR-100 dataset, when using 2, 3 or 4 layers with $h=2$ can be explained by the failure to initialise correctly in a transparent manner. This is supported by the evaluation of the accuracy, where all other models obtain similar accuracy to the uncalibrated models, these two cases yield about random accuracy. 

\paragraph{Multi-class calibration}
Our theory shows that $g$-layers, trained with NLL, optimize multi-class calibration, so far we have only evaluated the top-1 calibration. 
In this final set of experiments we show that our method indeed performs multi-class calibration. 

In the first experiment, we use the ResNet 110 model on CIFAR-100 (other networks / datasets are provided in the supplementary) and compare the uncalibrated network with a $g$-layer network (3 layer, $h=3$).
In \fig{top_ks_error} we show the calibration of the top-1 to top-4 classes, by the KS plots and the KS error for each class.
From these results it is clear that (a) $g$-layers significantly reduce the KS-error for all top-$k$ classes; (b) that the top (few) classes have by far the most influence on multi-class calibration metrics, \eg multi-class ECE~\cite{kull2019beyond} takes the average over all classes.

In the second experiment we evaluate the average top-10 KS error for all datasets and network architectures used. 
We compare the uncalibrated network with a $g$-layer network (3 layers, $h=3$). 
The results are in \tab{ECE_top10}.
From the results we observe that the calibrated network has a significant lower error, than the uncalibrated network. Based on these results from both experiments we conclude that $g$-layers indeed perform multi-class calibration.

\section{Conclusion}
The analysis in this paper gives broader conditions than previously known for a 
classifier such as a neural network to be correctly calibrated, ensuring that
the network can be correctly calibrated during training, after early-stopping or 
through post-hoc calibration. 
This provides a theoretical basis for post-hoc calibration schemes.

In this paper we have also introduced $g$-layers, a post-hoc calibration network. 
It consists of a series of transparent dense prediction layers, concatenated to a pre-trained network.
These $g$-layers are optimised using negative log likelihood training.
Experimentally we have shown that $g$-layers obtain excellent calibration performance, both when evaluated for the Top-1 class as well as evaluated for multi-classes, across a large set of networks and datasets. 
We intend to study techniques to improve generalization with respect to calibration as a future work.
\appendices
\setcounter{table}{0}
\renewcommand{\thetable}{A.\arabic{table}}
\setcounter{figure}{0}
\renewcommand{\thefigure}{A.\arabic{figure}}

We first provide the proofs of the results of our main paper and then provide additional results.
\section{Proofs of post-hoc calibration}
 
\subsection{Lemma about change of variables}
The following result will be useful.  It is perhaps relatively obvious, but worth stating
exactly.

\begin{lemma}
\label{lem:expectation-transfer}
Let $X$ be a random variable with values in $\domx$ and $f:\domx \rightarrow \dom_Z$ be a 
measurable function.  Let $h: \dom_Z \rightarrow \R$ be a measurable function.
If $Z = f(X)$, then
\[
E_{x\sim X} \, h(f(x)) = 
E_{z\sim Z} \,h(z)~.
\]
\end{lemma}
The proof follows from the definition of expected value, using a simple change of
variables.
\begin{proof} (sketch)
The expectations may be written as
\begin{align*}
 E_{x\sim X} \, h(f(x)) &= \int h(f(x))\, d{\mu_X} \\
 E_{z\sim Z}\, h(z) &= \int h(z) \,d{\mu_Z} ~,
 \end{align*}
 where $\mu_X$ and $\mu_Z$ are probability measures on $\domx$ and $\dom_Z$.
The desired equality then follows from a change of variables. \qed
\end{proof}
The lemma can be given a more informal but more intuitive proof as follows.
The expected value $E_{x\sim X} \, h(f(x))$ can be computed by
sampling $x$ from $X$ and taking the mean of the values $h(f(x))$.
In the limit as the number of samples increases, this mean converges to the
$E_{x\sim X} \, h(f(x))$.

Similarly, $E_{z\sim Z}\, h(z)$ is obtained by sampling
from $Z$ and computing the mean of the values $h(z)$.  However, random
samples from $Z$ are obtained by sampling $x \sim X$ for then $z = f(x)$
is a sample from the distribution $Z$.  Hence, the two expectations
give the same result.

\subsection{Submersions}
We are interested in submersions from $\R^m$ to $\Delta^{n-1}$,
the standard open simplex.  

%
\begin{proposition}
\label{pro:submersion}
Let 
$q: \R^m \rightarrow \Delta^{n-1}$ be a submersion.
If $~\sum_{y=1}^n w_y \, \partial q_y / \partial z_j = 0$,
then $w_y$ is a constant for all $y$.
\end{proposition}

\begin{proof}
Since $\Delta^{n-1}$ has dimension $n-1$,  if $q$
is a submersion,  the Jacobian
$\partial q_y(z) / \partial z_j$ has rank $n-1$.  
Since $\sum_{y=1}^n q_y(z) = 1$, taking derivatives gives
$\sum_{y=1}^n \partial q_y(z) / \partial z_j = 0$ for all $j$.
Written in terms of matrices, with $\m J = \partial q_y(z) / \partial z_j$
this says that $\v 1\tr \m J = 0$.  Further, since $\m J$ has rank $n-1$, 
if $\v w \tr \m J = 0$ then $\v w = \alpha \v 1$.
\end{proof}

\subsection{Negative-logarithm loss and calibration}

The following theorem is a known (in some form) property of the Negative-Logarithm
cost function.  The paper~\cite{bishop1994mixture} gives the essential idea of the proof,
but the theorem is not stated formally there.

\begin{theorem}
\label{thm:calibration-theorem}
Consider joint random variables $(Z, Y)$, taking values
in $\R^m$ and $\calY$ respectively.
Let $q: \R^m \rightarrow \Delta^{n-1}$ be a submersion.
Define the loss
\begin{align*}
L(q,  Z, Y) &= -E_{(z, y)\sim(Z, Y)} \log \big(q_y(z)\big) ~.
\end{align*}
If 
\[
\text{\rm id} = \argmin_{g:\R^m \rightarrow \R^m} L(q\circ g, Z, Y)
\]
then
\[
P(y ~|~ z) = q_y (z) ~.
\]
\end{theorem}

\begin{proof}
The assumption in this theorem is that the value of the 
loss function cannot be reduced by applying some function $g:\R^m \rightarrow \R^m$.
We investigate what happens
to the function $L(q, X, Y)$ when $q$ is replaced by the composition $q \circ g$, 
where $g:\R^m \rightarrow \R^m$ is
some function.  We compute
\begin{align}
\begin{split}
\label{eq:gofXK}
-L(q\circ g, \,Z, Y) &= E_{(z, y)\sim (Z, Y)} \log(q_y(g(z))) \\
 &= \int\sum_{y=1}^n \, p(z, y) \log(q_y(g(z)))\,dz   \\
\end{split}
\end{align}
%

We wish to optimize \eq{gofXK} over $g:\R^m\rightarrow \R^m$.  
Let $z \in \R^m$ and 
\[
z = (z_1, \ldots, z_m) = (g_1(z), \ldots, g_m(z)) = g(z) ~,
\]
which holds when $g$ is the identity function.

The Euler-Lagrange equation concerns a functional of the form $\int F(z, g, g') \, dz$, and 
if derivatives $g'$ do not appear in the functional, then 
the minimum (with respect to $g$) is attained when the Euler-Lagrange equation holds for every $j$:
\[
\frac{\partial F}{\partial g_j} = 0 ~.
\]
Since here 
\begin{align*}
F(z, g, g') &=\sum_{y=1}^n  p(z, y) \log (q_y(g(z)) )~, \\
&= \sum_{y=1}^n  p(z, y) \log (q_y(g_1, \ldots, g_m) )
\end{align*}
we compute
\[
\frac{\partial F}{\partial g_j} =\sum_{y=1}^n  p(z, y) \frac{\partial q_y/\partial z_j}{q_y(g_1, \ldots, g_m)} = 0 ~,
\]
where $\partial q_y /\partial z_j$ means the partial derivative of $q_y$ with respect to its $j$-th component.
This shows
%
%
\begin{align*}
\sum_{y=1}^n \frac{ p(z, y)} {q_y(z)} ~~~ \frac{\partial q_y}{ \partial z_j} = 0 ~.
\end{align*}
However, from \pro{submersion}, this implies that 
$p(z, y) / q_y(z) = c$, a constant, so $p(z, y) = c\, q_y(z)$ for all $y$.
However, since $\sum_y q_y(z) = 1$ and $\sum_y p(z, y) = p(z)$, this gives
$p(z, y) = p(z)\,  q_y(z)$, or $p(y ~|~ z) = q_y(z)$, as required.
This completes the proof of \thm{calibration-theorem}. \
\qed
\end{proof}

A simple rewording of this theorem (changing the names of the variables)
gives the following statement, which is essentially a formal statement
of a result stated in~\cite{bishop1994mixture}.

\begin{corollary}
\label{cor:calibration-corollary}
Consider joint random variables $(X, Y)$, taking values
in $\domx$ and $\calY$ respectively, where $\domx$ is some Cartesian space.
Let $f: \domx \rightarrow \Delta^{n-1}$ be a function.
Define the loss
\begin{align*}
L(f,  X, Y) &= -E_{(x, y)\sim(X, Y)} \log \big(f_y(x)\big) ~.
\end{align*}
If 
\[
f = \argmin_{\hat f:\domx \rightarrow \Delta^{n-1}} L(q\circ \hat f, X, Y)
\]
then
\[
P(y ~|~ x) = f_y (x) ~.
\]
\end{corollary}

This corollary follows directly from \thm{calibration-theorem} since
if $f$ minimizes the cost function over all functions, then it optimizes 
the cost over all functions $f\circ g$, where $g: \domx \rightarrow \domx$.

However, when $\domx$ is a high-dimensional space (such as a space of
images), then it may be a very difficult task to find the optimum function
$f$ exactly.  Fortunately, a much less stringent condition is enough to
ensure the conclusion of the theorem, and that the network ($f$) is calibrated.
\begin{corollary}
\label{thm:calibration-theorem-sup}
Consider joint random variables $(X, Y)$, taking values
in $\domx$ and $\calY$ respectively.
Let $f:\domx \rightarrow \R^m$,
and let $q: \R^m \rightarrow \Delta^{n-1}$ be a submersion.
Define the loss
\begin{align*}
L(q\circ f,  X, Y) &= -E_{(x, y)\sim(X,Y)} \log\big(q_y(f(x))\big) ~.
\end{align*}
If $f$ is is optimal with
respect to recalibration, for this cost function,
then $P(y ~|~ f(x)) =  q_y(f(x))$.
\end{corollary}

\begin{proof}
Let $Z = f(X)$, and for any $x$ define $f(x) = z$.  Then, from
\lem{expectation-transfer}, 
\begin{align*}
L(q\circ f,  X, Y) &= -E_{(x, y)\sim(X,Y)} \log\big(q_y(f(x))\big)  \\
 &= -E_{(z, y)\sim(Z,Y)} \log\big(q_y(z)\big)  \\
 &= L(q, Z, Y) ~.                    
\end{align*}
Then, according to \thm{calibration-theorem}, if replacing
$q$ by $q\circ g$ will not decrease the value of the loss function
(which is the definition that $f$ is optimal with respect to recalibration),
we may conclude that 
\[
P(y ~|~ z) = q_y (z) ~.
\]
as required.
\qed
\end{proof}

\section{Multiclass and classwise calibration}

We make the usual assumption of random variables $X$ and $Y$.
Suppose that a function ${f:\domx \rightarrow \Delta^{n-1}}$ is 
multiclass calibrated, which means that $P(y ~|~ z) = z_y$, where
$z = f(x)$.  We wish to show that it is classwise calibrated,
meaning $P(y ~|~ z_y) = z_y$, and also that it is calibrated 
for top-$r$ and within-top-$r$ calibration.
It was stated in~\cite{kull2019beyond} that classwise calibration, and calibration
for the top class are ``weaker'' concepts of calibration, but no
justification was given there.  Hence, we fill that gap in the theorem below.

The proof is not altogether trivial, since certainly $P(y ~|~ z_y)$
is not equal to $P(y ~|~ z)$ in general.
Neither does it follow from the fact that
$Z = z$ implies $Z_y = z_y$.

First, we change notation just a little.  Let $\haty$ be the so-called $n$-dimensional one-hot
vector of $y$, namely an indicator vector such that 
$\haty_k = 1$ if $y = k$ and $0$ otherwise.  Then the condition 
for multi-class calibration is 
\[
P(\haty_y = 1 ~|~ z_y = \sigma) = \sigma
\]

\paragraph{The top-$r$ prediction. }
We wish also to talk about calibration of the top-scoring class predictions.
Suppose a classifier $f$ is given with values in $\Delta^{n-1}$
and let $y$ be the ground truth label.
Let us use $z^r$ to denote the $r$-th top score (so $z^1$ 
would denote the top score). Note that an upper index, such as  in $z^r$ here
represents the $r$-th top value, whereas lower indices, such as $z_y$
represent the $y$-th class.
Similarly, define
%
$\haty^r$ to be $1$ if the $r$-th top predicted class is the correct
(ground-truth) choice, and $0$ otherwise.
The network is calibrated for the top-$r$ predictor if for all scores $\sigma$,
\begin{equation}
\label{eq:top-r-calibrated}
P(\haty^r = 1 ~|~ z^r = \sigma) = \sigma  ~.
\end{equation}
%
In words, the conditional probability that the top-$r$-th choice of the network is
the correct choice, is equal
to the $r$-th top score.  

Similarly, one may consider probabilities that a datum belongs to one of the top-$r$ scoring classes.
The classifier is calibrated for being within-the-top-$r$ classes if
\begin{equation}
\label{eq:within-top-r-calibrated}
P\big(\textstyle\sum_{s=1}^r \haty^s = 1 ~\big|~ \sum_{s=1}^r z^s = \sigma\big) = \sigma  ~.
\end{equation}
Here, the sum on the left is $1$ if the ground-truth label is among the
top $r$ choices, $0$ otherwise, and the sum on the right is the sum
of the top $r$ scores.
 

\begin{theorem}
Suppose random variables $X$ and $Y$ defined on $\domx$ and
$\calY$ respectively, and let $f:\domx \rightarrow \Delta^{n-1}$ be
a measurable function.  Suppose that $P( y ~|~ z) = z_y$, where $z = f(x)$.
Then $f$ is classwise calibrated, and also calibrated for
top-$r$ and within-top-$r$ classes, as defined by \eq{top-r-calibrated}
and \eq{within-top-r-calibrated}.
\end{theorem}

\begin{proof}
We assume that $f$ is multiclass calibrated,
so that $P(\haty_y = 1 ~|~ z) = z_y$.  
First, we observe that 
\begin{equation}
\label{eq:Pkz}
P(\haty_y = 1, z) = P(\haty_y = 1 ~|~ z) \, P(z) = z_y P(z) ~.
\end{equation}
Then,
\begin{align*}
P(\haty_y = 1 ~|~ z_y=\sigma) &= P(\haty_y = 1, z_y=\sigma) ~/~ P(z_y = \sigma) \\
                     &= \int_{z_y=\sigma} P(\haty_y = 1, z) \, dz ~/~ P(z_y = \sigma)
\end{align*}
where the integral marginalizes over all values of $z$ with $y$-th entry equal to $\sigma$.
Continuing, using \eq{Pkz} gives
\begin{align*}
P(\haty_y = 1 ~|~ z_y=\sigma) &= \int_{z_y=\sigma} z_y P(z) \, dz ~/~ P(z_y = \sigma) \\
                     &= z_y \int_{z_y=\sigma} P(z) \, dz ~/~ P(z_y = \sigma) \\
                     &= z_y P(z_y = \sigma) ~/~ P(z_y = \sigma) = z_y
\end{align*}
which proves that $f$ is classwise calibrated.

Next, we show that $f$ is top-$r$ calibrated.  The proof is much the same,
using top indices rather than lower indices.
Analogously to \eq{Pkz}, we have
\begin{equation}
\label{eq:Pkzr}
P(\haty^r=1, z) = P(\haty^r=1 ~|~ z)  \, P(z) = z^r P(z) ~.
\end{equation}
This equation uses the equality $P(\haty^r=1 ~|~ z) = z^r$.  To see this, fix $z$,
and let $y$ be the index of the $r$-th highest entry of $z$.
Then $z^r = z_y$ and $\haty^r = \haty_y$.  Then 
$P(\haty^r=1 ~|~ z) = P(\haty_y = 1 ~|~ z) = z_y = z^r$.

Then,
\begin{align*}
P\big(\haty^r  = 1 ~|~ z^r = \sigma\big) 
   &= P\big(\haty^r  = 1,  z^r= \sigma\big) / P(z^r = \sigma) \\
   & =\int_{z^r=\sigma} P\big(\haty^r  = 1,  z\big)\, dz ~/~ P(z^r = \sigma) \\
   &\stackrel{\text{\eq{Pkzr}}}{=} z^r \int_{z^r=\sigma} P(z) \, dz ~/~ P(z^r = \sigma) \\
   &= z^r  P(z^r = \sigma)  ~/~  P(z^r = \sigma) \\
   &= z^r ~.
\end{align*}
Here, the integral is over all $z$ such that $z^r = \sigma$.  This shows that
$f$ is top-$r$ calibrated.

Finally, we prove within-top-$r$ calibration.  Refer to \eq{within-top-r-calibrated}, let $\sigma$ be fixed, and let $z$ be some vector such 
that $\sum_{s=1}^r z^s = \sigma$.
 Since for $s=1, \ldots, r$ the events
$\haty^s = 1$ are mutually exclusive, it follows that
\begin{align*}
P\big(\textstyle\sum_{s=1}^r \haty^s = 1 ~\big|~ z \big) ~=~
\textstyle\sum_{s=1}^r P\big( \haty^s ~=~ 1 ~\big|~ z \big)= \textstyle\sum_{s=1}^r z^s = \sigma ~.
\end{align*}
This equality $P\big(\textstyle\sum_{s=1}^r \haty^s = 1 ~\big|~ z \big)  = \sigma$
will hold for any $z$ such that $\sum_{s=1}^r z^s = \sigma$.  It follows that
\[
P\big(\textstyle\sum_{s=1}^r \haty^s = 1 ~\big|~ 
\sum_{s=1}^r z^s = \sigma \big) = \sigma ~,
\]
as required.

Note the following justification for this last step.  If some
random variables $A$ and $Z$
satisfy $P(A = a ~|~ Z = z) = \sigma$ (a constant) for all $z$ in some class $C$,
then $P(A = a ~|~ z \in C) = \sigma$.  For, the assumption implies 
that $P(A = a, Z = z) = \sigma P(Z = z)$.  Now, integrating for $z \in C$ gives
$P( A = a, Z \in C) = \sigma P(Z \in C)$, and hence $P(A = a ~|~ Z \in C) = \sigma$.
\qed
\end{proof}

What this theorem is saying, for instance, is that the probability
that the correct classification lies within the top $2$ (or $r$) scoring classes,
given that the sum of these two scores is $\sigma$, is equal to the
sum of the two top scores. 

The theorem can easily be extended to any set of classes, to show that 
if the classifier $f$ is multiclass calibrated and $S$ is any set of labels, that
\begin{equation}
P\big(\textstyle\sum_{s\in S} \haty^s = 1 ~\big|~ \sum_{s\in S} z^s = \sigma\big) = \sigma  ~,
\end{equation}
and 
\begin{equation}
P\big(\textstyle\sum_{s\in S} \haty_s = 1 ~\big|~ \sum_{s\in S} z_s = \sigma\big) = \sigma  ~.
\end{equation}

\section{Additional Results}
In this section we provide additional results complementary to the results in the main paper.

For our experimental validation we calibrate deep convolutional neural networks trained on CIFAR-10/CIFAR-100~\cite{krizhevsky2009learning}, SVHN~\cite{netzer2011reading} and ILSVRC'12~\cite{ILSVRC15} datasets. 
For the base network $f$ we use pre-trained models of different architectures: ResNet~\cite{he2016deep}, ResNet Stochastic Depth~\cite{huang2016eccv}, DenseNet~\cite{huang2017densely}, and Wide ResNet~\cite{zagoruyko2016wide}.
For these experiments we use the pre-trained models also used in~\cite{kull2019beyond}\footnote{Pre-trained models are obtained from: \url{https://github.com/markus93/NN_calibration}.}.
We train the proposed $g$-layers on the validation set (not used for training base networks) which we denote \emph{calibration set}.
The models are then evaluated on the unseen test set.

\subsection{Calibration Error and Accuracy}
\renewcommand{\bst}[1]{\textbf{#1}}
\begin{table*}[p]
\begin{subtable}[b]{\textwidth}
	\centering
	\resizebox{.9\textwidth}{!}{
	\scriptsize
	\begin{tabular}{llcccc|ccccc}
		\toprule
                Dataset                   & Base Network              &         Unc &          TS &          MS &         Dir &         \multicolumn{5}{c}{\textbf{$g$-Layers}}\\\		                          &                           &             &             &             &             &          D1 &          D2 &          D3 &          D4 &          D5 \\\midrule
		CIFAR-10                  & ResNet 110                &       4.751 &  \bst{0.917}&       0.988 &       1.076 &       0.924 &       0.990 &       0.954 &       1.116 &       1.066 \\
		                          & ResNet 110 SD             &       4.103 &       0.362 &       0.331 &       0.368 &       0.317 &       0.378 &       0.342 &       0.307 &  \bst{0.188}\\
		                          & Wide ResNet 32            &       4.476 &       0.296 &  \bst{0.284}&       0.313 &       0.320 &       0.296 &       0.351 &       0.337 &       0.420 \\
		                          & DensNet 40                &       5.493 &       0.900 &       0.897 &       0.969 &       0.911 &       1.026 &  \bst{0.669}&       1.679 &       1.377 \\\midrule
		SVHN                      & ResNet 152 SD             &       0.853 &  \bst{0.553}&       0.572 &       0.588 &       0.593 &       0.561 &       0.579 &       0.588 &       0.564 \\\midrule
		CIFAR-100                 & ResNet 110                &      18.481 &       1.489 &       2.541 &       2.335 &       1.359 &       1.618 &  \bst{0.526}&       1.254 &           - \\
		                          & ResNet 110 SD             &      15.833 &       0.748 &       2.158 &       1.901 &  \bst{0.589}&       1.165 &       0.848 &       0.875 &           - \\
		                          & Wide ResNet 32            &      18.784 &       1.130 &       2.821 &       2.000 &       0.831 &  \bst{0.757}&       1.900 &       0.857 &           - \\
		                          & DensNet 40                &      21.157 &       0.305 &       2.709 &       0.775 &       0.249 &  \bst{0.188}&       0.199 &       0.203 &           - \\\midrule
		ILSVRC'12                 & ResNet 152                &       6.544 &       0.792 &       5.355 &       4.400 &       0.776 &  \bst{0.755}&       0.849 &       0.757 &  -\\
		                          & DensNet 161               &       5.721 &  \bst{0.744}&       4.333 &       3.824 &       0.881 &       1.091 &       0.780 &       1.105 &       - \\\bottomrule
	\end{tabular}
	}
	\caption{Calibration error on the test set, measured with KS error (Top-1 in \%).}
	\label{tab:testKSE}
\end{subtable}\vspace{5mm}

\begin{subtable}[b]{\textwidth}
	\centering
	\resizebox{.9\textwidth}{!}{
	\scriptsize
	\begin{tabular}{llcccc|ccccc}
		\toprule
                Dataset                   & Base Network              &         Unc &          TS &          MS &         Dir &         \multicolumn{5}{c}{\textbf{$g$-Layers}}\\\		                          &                           &             &             &             &             &          D1 &          D2 &          D3 &          D4 &          D5 \\\midrule
		CIFAR-10                  & ResNet 110                &   \bst{93.6}&   \bst{93.6}&        93.5 &        93.5 &   \bst{93.6}&        93.5 &        93.5 &        93.5 &        93.4 \\
		                          & ResNet 110 SD             &        94.0 &        94.0 &   \bst{94.2}&   \bst{94.2}&   \bst{94.2}&        94.1 &        94.1 &        94.1 &        94.0 \\
		                          & Wide ResNet 32            &        93.9 &        93.9 &   \bst{94.2}&   \bst{94.2}&   \bst{94.2}&   \bst{94.2}&        94.1 &        94.1 &        94.0 \\
		                          & DensNet 40                &        92.4 &        92.4 &        92.5 &        92.5 &   \bst{92.6}&        92.4 &        92.4 &        92.3 &        92.3 \\\midrule
		SVHN                      & ResNet 152 SD             &   \bst{98.2}&   \bst{98.2}&        98.1 &   \bst{98.2}&   \bst{98.2}&   \bst{98.2}&   \bst{98.2}&   \bst{98.2}&        98.1 \\\midrule
		CIFAR-100                 & ResNet 110                &        71.5 &        71.5 &   \bst{71.6}&   \bst{71.6}&        71.5 &        71.5 &        71.5 &        71.5 &           - \\
		                          & ResNet 110 SD             &        72.8 &        72.8 &   \bst{73.5}&        73.1 &        73.0 &        72.9 &        73.2 &        73.2 &           - \\
		                          & Wide ResNet 32            &        73.8 &        73.8 &   \bst{74.0}&   \bst{74.0}&        73.8 &        73.9 &        73.9 &        73.8 &           - \\
		                          & DensNet 40                &        70.0 &        70.0 &   \bst{70.4}&        70.2 &        70.0 &        70.1 &        70.2 &        70.2 &           - \\\midrule
		ILSVRC'12                 & ResNet 152                &   \bst{76.2}&   \bst{76.2}&        76.1 &   \bst{76.2}&   \bst{76.2}&   \bst{76.2}&   \bst{76.2}&   \bst{76.2}&   -\\
		                          & DensNet 161               &        77.0 &        77.0 &   \bst{77.2}&   \bst{77.2}&        77.0 &        77.1 &        77.0 &        77.0 &        - \\\bottomrule
	\end{tabular}
	}
	\caption{Classification accuracy on the test set, measured with top-1 accuracy (in \%).}
	\label{tab:testACC}
\end{subtable}\vspace{5mm}

\begin{subtable}[b]{\textwidth}
	\centering
	\resizebox{.9\textwidth}{!}{
	\scriptsize
	\begin{tabular}{llcccc|ccccc}
		\toprule
                Dataset                   & Base Network              &         Unc &          TS &          MS &         Dir &         \multicolumn{5}{c}{\textbf{$g$-Layers}}\\\		                          &                           &             &             &             &             &          D1 &          D2 &          D3 &          D4 &          D5 \\\midrule
		CIFAR-10                  & ResNet 110                &       1.102 &       0.979 &  \bst{0.975}&       0.976 &       0.979 &       0.976 &       0.978 &       0.983 &       0.988 \\
		                          & ResNet 110 SD             &       0.981 &       0.874 &  \bst{0.866}&       0.867 &  \bst{0.866}&       0.870 &       0.875 &       0.875 &       0.887 \\
		                          & Wide ResNet 32            &       1.047 &       0.924 &       0.890 &       0.888 &       0.890 &  \bst{0.887}&       0.896 &       0.897 &       0.900 \\
		                          & DensNet 40                &       1.274 &       1.100 &  \bst{1.096}&       1.097 &       1.099 &       1.098 &       1.102 &       1.113 &       1.121 \\\midrule
		SVHN                      & ResNet 152 SD             &       0.297 &  \bst{0.291}&       0.298 &       0.293 &  \bst{0.291}&  \bst{0.291}&  \bst{0.291}&       0.292 &       0.293 \\\midrule
		CIFAR-100                 & ResNet 110                &       0.453 &       0.392 &  \bst{0.391}&  \bst{0.391}&       0.392 &       0.392 &       0.392 &       0.393 &           - \\
		                          & ResNet 110 SD             &       0.418 &       0.367 &  \bst{0.361}&       0.363 &       0.365 &       0.366 &       0.365 &       0.364 &           - \\
		                          & Wide ResNet 32            &       0.432 &       0.355 &  \bst{0.351}&       0.354 &       0.355 &       0.354 &       0.354 &       0.354 &           - \\
		                          & DensNet 40                &       0.491 &       0.401 &  \bst{0.400}&  \bst{0.400}&       0.401 &  \bst{0.400}&       0.401 &       0.401 &           - \\\midrule
		ILSVRC'12                 & ResNet 152                &       0.034 &  \bst{0.033}&  \bst{0.033}&  \bst{0.033}&  \bst{0.033}&  \bst{0.033}&  \bst{0.033}&  \bst{0.033}&  -\\
		                          & DensNet 161               &       0.033 &  \bst{0.032}&  \bst{0.032}&  \bst{0.032}&  \bst{0.032}&  \bst{0.032}&  \bst{0.032}&  \bst{0.032}&  -\\\bottomrule
	\end{tabular}
	}
	\caption{Calibration error on the test set, measured with the Brier error (multiplied by constant 100).}
	\label{tab:testBRI}
\end{subtable}\vspace{5mm}

\caption{Overview of calibration error and accuracy on the (unseen) test set -- part 1.}\label{tab:full_overview_results_part1}
\end{table*}

\begin{table*}
\begin{subtable}[p]{\textwidth}
	\centering
	\resizebox{.9\textwidth}{!}{
	\scriptsize
	\begin{tabular}{llcccc|ccccc}
		\toprule
                Dataset                   & Base Network              &         Unc &          TS &          MS &         Dir &         \multicolumn{5}{c}{\textbf{$g$-Layers}}\\\		                          &                           &             &             &             &             &          D1 &          D2 &          D3 &          D4 &          D5 \\\midrule
		CIFAR-10                  & ResNet 110                &       4.750 &       1.132 &       1.052 &       1.144 &       1.130 &  \bst{0.997}&       1.348 &       1.152 &       1.219 \\
		                          & ResNet 110 SD             &       4.113 &       0.555 &       0.599 &       0.739 &       0.807 &       0.809 &       0.629 &       0.674 &  \bst{0.503}\\
		                          & Wide ResNet 32            &       4.505 &       0.784 &       0.784 &       0.796 &  \bst{0.616}&       0.661 &       0.634 &       0.669 &       0.670 \\
		                          & DensNet 40                &       5.500 &       0.946 &       1.006 &       1.095 &       1.101 &       1.037 &  \bst{0.825}&       1.729 &       1.547 \\\midrule
		SVHN                      & ResNet 152 SD             &       0.862 &       0.607 &       0.616 &       0.590 &       0.638 &  \bst{0.565}&       0.589 &       0.604 &       0.648 \\\midrule
		CIFAR-100                 & ResNet 110                &      18.480 &       2.380 &       2.718 &       2.896 &       2.396 &       2.334 &  \bst{1.595}&       1.792 &           - \\
		                          & ResNet 110 SD             &      15.861 &  \bst{1.214}&       2.203 &       2.047 &       1.219 &       1.405 &       1.423 &       1.298 &           - \\
		                          & Wide ResNet 32            &      18.784 &       1.472 &       2.821 &       1.991 &       1.277 &  \bst{1.096}&       2.199 &       1.743 &           - \\
		                          & DensNet 40                &      21.156 &       0.902 &       2.709 &       0.962 &       0.927 &       0.644 &       0.562 &  \bst{0.415}&           - \\\midrule
		ILSVRC'12                 & ResNet 152                &       6.543 &       2.077 &       5.353 &       4.491 &       2.025 &       2.051 &   \bst{1.994} &       2.063 &  -\\
		                          & DensNet 161               &       5.720 &       1.942 &       4.333 &       3.926 &       1.952 &       1.978 &  \bst{1.913}&       1.931 &       - \\\bottomrule
	\end{tabular}
	}
	\caption{Calibration error on the test set, measured with ECE error (Top-1 in \%).}
	\label{tab:testECE}
\end{subtable}

\begin{subtable}[b]{\textwidth}
	\centering
	\resizebox{.9\textwidth}{!}{
	\scriptsize
	\begin{tabular}{llcccc|ccccc}
		\toprule
                Dataset                   & Base Network              &         Unc &          TS &          MS &         Dir &         \multicolumn{5}{c}{\textbf{$g$-Layers}}\\\		                          &                           &             &             &             &             &          D1 &          D2 &          D3 &          D4 &          D5 \\\midrule
		CIFAR-10                  & ResNet 110                &       1.870 &  \bst{0.460}&       0.481 &       0.481 &       0.473 &       0.479 &       0.477 &       0.499 &       0.463 \\
		                          & ResNet 110 SD             &       1.627 &       0.200 &       0.186 &       0.192 &       0.199 &       0.189 &       0.180 &       0.175 &  \bst{0.145}\\
		                          & Wide ResNet 32            &       1.759 &       0.202 &       0.193 &       0.192 &  \bst{0.162}&       0.174 &       0.203 &       0.191 &       0.218 \\
		                          & DensNet 40                &       2.168 &       0.386 &       0.382 &       0.408 &       0.401 &       0.424 &  \bst{0.293}&       0.663 &       0.562 \\\midrule
		SVHN                      & ResNet 152 SD             &       0.304 &       0.219 &  \bst{0.217}&       0.222 &       0.231 &       0.220 &       0.225 &       0.223 &       0.218 \\\midrule
		CIFAR-100                 & ResNet 110                &       5.874 &       0.829 &       1.024 &       0.977 &       0.794 &       0.857 &  \bst{0.594}&       0.711 &           - \\
		                          & ResNet 110 SD             &       5.212 &       0.349 &       0.748 &       0.696 &  \bst{0.288}&       0.429 &       0.351 &       0.375 &           - \\
		                          & Wide ResNet 32            &       6.223 &       0.599 &       1.005 &       0.780 &       0.516 &  \bst{0.491}&       0.817 &       0.558 &           - \\
		                          & DensNet 40                &       6.950 &       0.243 &       0.871 &       0.353 &       0.210 &  \bst{0.179}&       0.254 &       0.292 &           - \\\midrule
		ILSVRC'12                 & ResNet 152                &       2.078 &       0.570 &       1.691 &       1.535 &       0.569 &  \bst{0.560}&       0.565 &       0.561 &  -\\
		                          & DensNet 161               &       1.847 &       0.531 &       1.385 &       1.356 &       0.563 &       0.614 &  \bst{0.523}&       0.616 &       - \\\bottomrule
	\end{tabular}
	}
	\caption{Calibration error on the test set, measured in average Top-5 KS error (in \%).}
	\label{tab:testKSR-5}
\end{subtable}\vspace{5mm}

\begin{subtable}[b]{\textwidth}
	\centering
	\resizebox{.9\textwidth}{!}{
	\scriptsize
	\begin{tabular}{llcccc|ccccc}
		\toprule
                Dataset                   & Base Network              &         Unc &          TS &          MS &         Dir &         \multicolumn{5}{c}{\textbf{$g$-Layers}}\\\		                          &                           &             &             &             &             &          D1 &          D2 &          D3 &          D4 &          D5 \\\midrule
		CIFAR-10                  & ResNet 110                &       1.066 &       0.275 &       0.277 &       0.283 &       0.283 &       0.285 &       0.283 &       0.291 &  \bst{0.269}\\
		                          & ResNet 110 SD             &       0.918 &       0.118 &       0.113 &       0.114 &       0.117 &       0.113 &       0.108 &       0.105 &  \bst{0.088}\\
		                          & Wide ResNet 32            &       0.995 &       0.123 &       0.119 &       0.117 &  \bst{0.101}&       0.107 &       0.121 &       0.114 &       0.127 \\
		                          & DensNet 40                &       1.234 &       0.225 &       0.223 &       0.239 &       0.236 &       0.249 &  \bst{0.174}&       0.387 &       0.325 \\\midrule
		SVHN                      & ResNet 152 SD             &       0.188 &       0.134 &  \bst{0.133}&  \bst{0.133}&       0.142 &       0.134 &       0.138 &       0.137 &       0.134 \\\midrule
		CIFAR-100                 & ResNet 110                &       3.299 &       0.462 &       0.579 &       0.540 &       0.443 &       0.488 &  \bst{0.336}&       0.401 &           - \\
		                          & ResNet 110 SD             &       2.904 &       0.241 &       0.434 &       0.407 &  \bst{0.210}&       0.277 &       0.234 &       0.234 &           - \\
		                          & Wide ResNet 32            &       3.459 &       0.348 &       0.574 &       0.445 &       0.304 &  \bst{0.299}&       0.462 &       0.338 &           - \\
		                          & DensNet 40                &       3.877 &       0.181 &       0.514 &       0.219 &       0.158 &  \bst{0.128}&       0.169 &       0.183 &           - \\\midrule
		ILSVRC'12                 & ResNet 152                &       1.155 &       0.329 &       0.942 &       0.839 &       0.329 &  \bst{0.326}&       0.329 &       0.333 &  -\\
		                          & DensNet 161               &       1.040 &       0.316 &       0.783 &       0.751 &       0.334 &       0.363 &  \bst{0.306}&       0.361 &       - \\\bottomrule
	\end{tabular}
	}
	\caption{Calibration error on the test set, measured in average Top-10 KS error (in \%).}
	\label{tab:testKSR-10}
\end{subtable}\vspace{5mm}

\caption{Overview of calibration error and accuracy on the (unseen) test set -- part 2.}\label{tab:full_overview_results_part2}
\end{table*}
For this set of experiments, we train $g$-layers with different number of dense layers, in the range from 1--5.
The size of the hidden $g$-layers is fixed to 3 times the number of classes in the dataset plus 2, resulting in, 32, 302, and 3002 for CIFAR-10/SVHN, CIFAR-100, and ILSVRC'12 respectively. This relative large number is required for the transparent initialisation, which requires the number of dense units to be larger than double the number of classes.
In practice this means that the number of weights in $g$-layers scale cubic with the number of classes, \eg the 3-layer network for ILSVRC'12 contains 15M weights, while the 4-layer network has 24M.

The hyper-parameters for learning (learning rate and weight decay) are determined using 5-fold cross validation on the calibration set.
The best results are used to train the $g$-layers on the full calibration set.
We do cross validation for the single layer and 3 layer $g$-layer network, and use the parameters found for the 3 layer network also for $\{2, 4, 5\}$ layer networks.
For all $g$-layer models we use early stopping based on the negative log-likelihood of the current training set.

\paragraph{Results}
In this set of experiments we compare $g$-layer calibration to several other calibration methods, including: (1) Temperature scaling~\cite{guo2017calibration}; (2) MS-ODIR~\cite{kull2019beyond}; and (3)  Dir-ODIR~\cite{kull2019beyond}.

The results are presented in \tab{full_overview_results_part1} and \tab{full_overview_results_part2}, reporting:
\begin{enumerate}
    \setlength{\itemsep}{0pt}
    \item KS error~\cite{gupta-spline-iclr21} in \tab{testKSE};
    \item Classification accuracy in \tab{testACC};
    \item Brier score in \tab{testBRI};
    \item Expected Calibration error in \tab{testECE};    
    \item Average top-5 KS error \tab{testKSR-5}; and
    \item Average top-10 KS error \tab{testKSR-10}.
\end{enumerate}

From the \tab{testACC} we conclude that calibration does not hurt classification accuracy. The accuracy of the different models remains within $\pm .1$ absolute percent point of the accuracy of the uncalibrated network. 

Based on all results in \tab{full_overview_results_part1} and \tab{full_overview_results_part2} we conclude that the proposed $g$-layers provide effective calibrated models over a wide range of network depths.

\subsection{Top-$k$ KS Calibration Error}
\begin{figure*}[p]
    \centering
    \begin{subfigure}[b]{1.05\textwidth}
        \includegraphics[width=\textwidth]{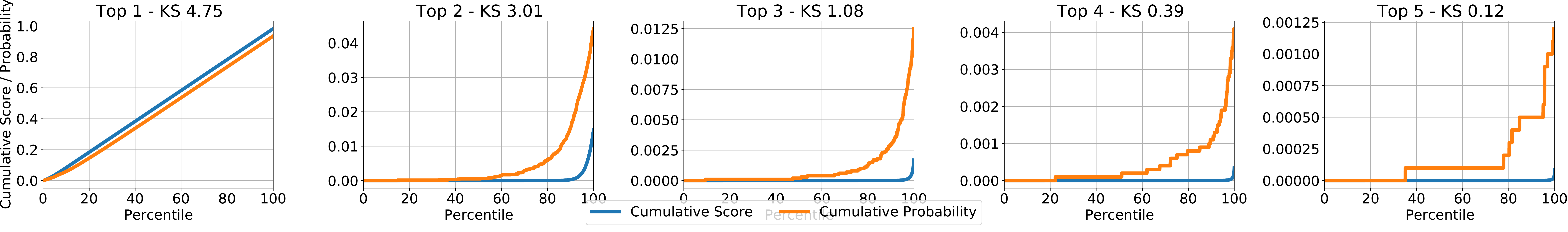}
        \includegraphics[width=\textwidth]{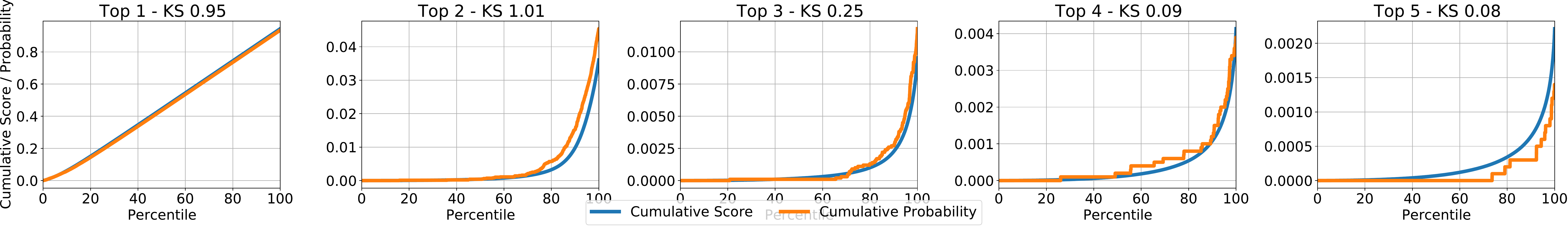}
        \caption{CIFAR-10 - ResNet 110: \textbf{top} Uncalibrated vs \textbf(bottom) Calibrated (3 $g$-layer)}
    \end{subfigure}
    \vspace{5mm}
    
    \begin{subfigure}[b]{1.05\textwidth}
        \includegraphics[width=\textwidth]{figures/iccv/c100_resnet110_plot_ks.pdf}
        \includegraphics[width=\textwidth]{figures/iccv/c100_resnet110_cal_D3_U0302_plot_ks.pdf}
        \caption{CIFAR-100 - ResNet 110: \textbf{top} Uncalibrated vs \textbf(bottom) Calibrated (3 $g$-layer)}
    \end{subfigure}
    \vspace{5mm}
    
    \begin{subfigure}[b]{1.05\textwidth}
        \includegraphics[width=\textwidth]{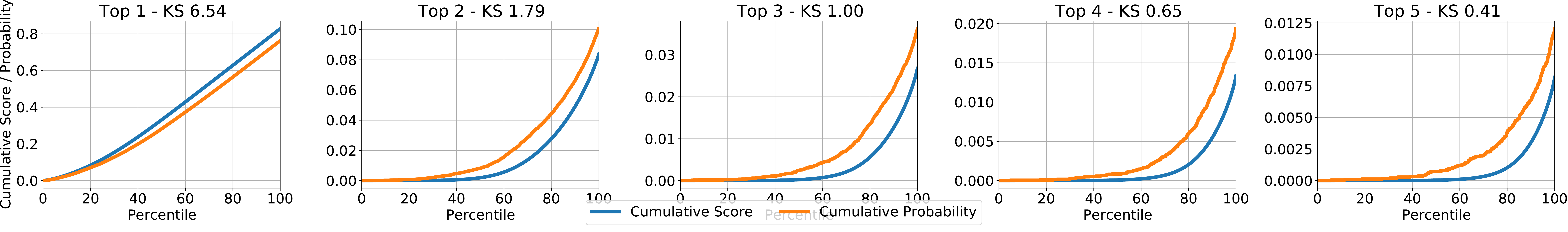}
        \includegraphics[width=\textwidth]{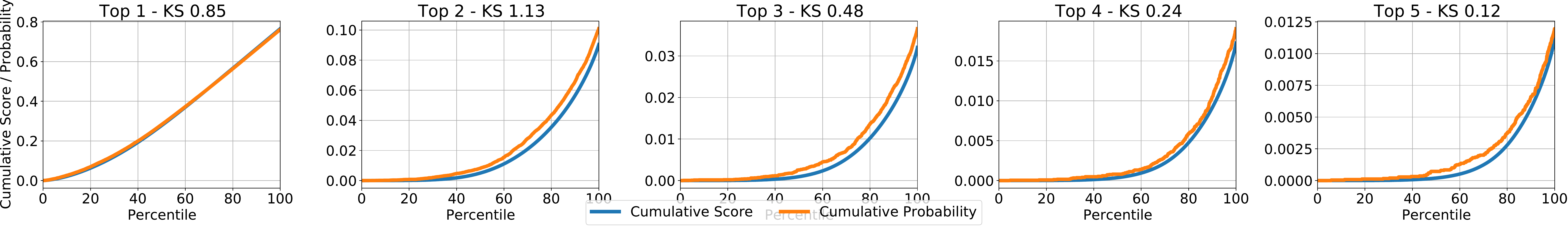}
        \caption{ImageNet - ResNet 152: \textbf{top} Uncalibrated vs \textbf(bottom) Calibrated (3 $g$-layer)}
    \end{subfigure}
    \caption{Calibration error for top-$k$ KS error, with $k = \{1, \ldots, 5\}$.}
    \label{fig:resnet_calibration_topk}
\end{figure*}
In this set of experiments we present the top-$k$ KS error for ResNet 110 on CIFAR-10 and CIFAR-100, and for ResNet 152 on ImageNet. 
We compare the uncalibrated network to a 3 $g$-layer network with 32/302, 3002 hidden units per layer (for 10/100/1000 classes).
The results are presented in~\fig{resnet_calibration_topk}.

\subsection{Number of parameters}
\begin{figure*}[p]
    \begin{subfigure}[b]{\textwidth}
        \centering
        \includegraphics[height=45mm]{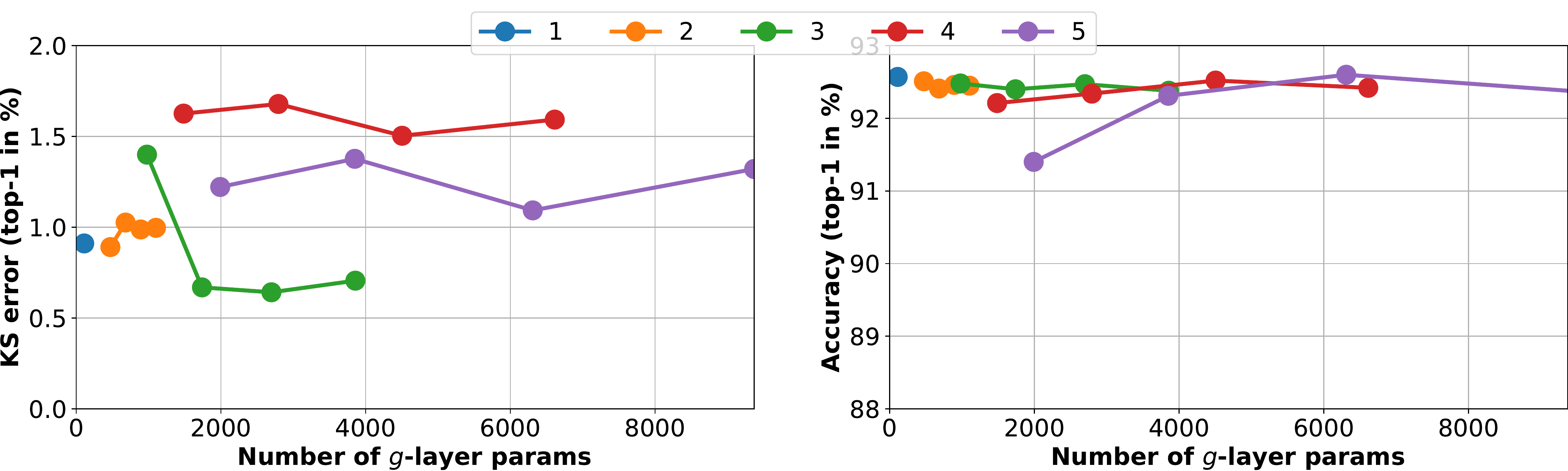}
        \caption{CIFAR-10 -- DenseNet 40}
    \end{subfigure}\vspace{5mm}
    
    \begin{subfigure}[b]{\textwidth}
        \centering
        \includegraphics[height=45mm]{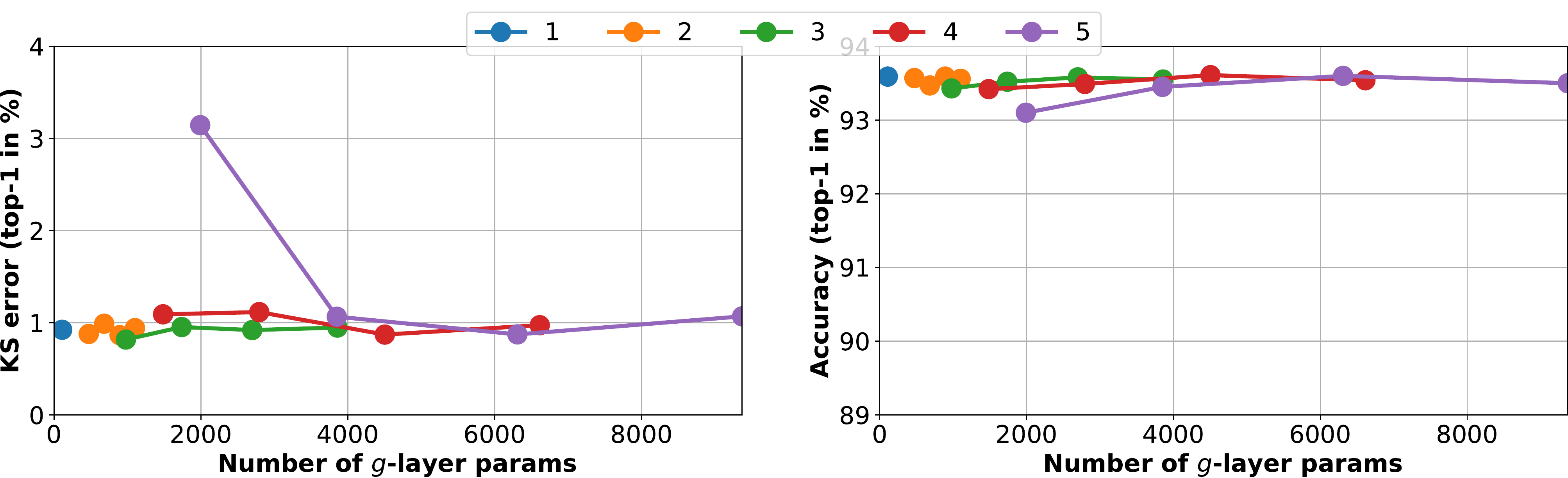}
        \caption{CIFAR-10 -- ResNet 110}
    \end{subfigure}\vspace{5mm}
    
    \begin{subfigure}[b]{\textwidth}
        \centering
        \includegraphics[height=45mm]{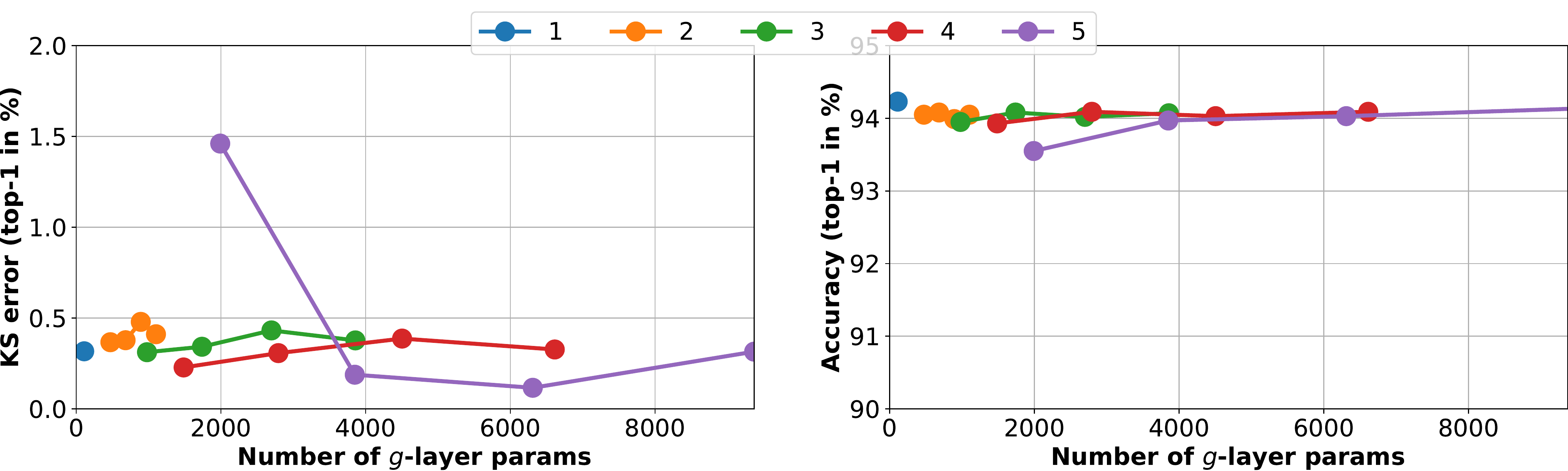}
        \caption{CIFAR-10 -- ResNet 110 SD}
    \end{subfigure}\vspace{5mm}
    
    \begin{subfigure}[b]{\textwidth}
        \centering
        \includegraphics[height=45mm]{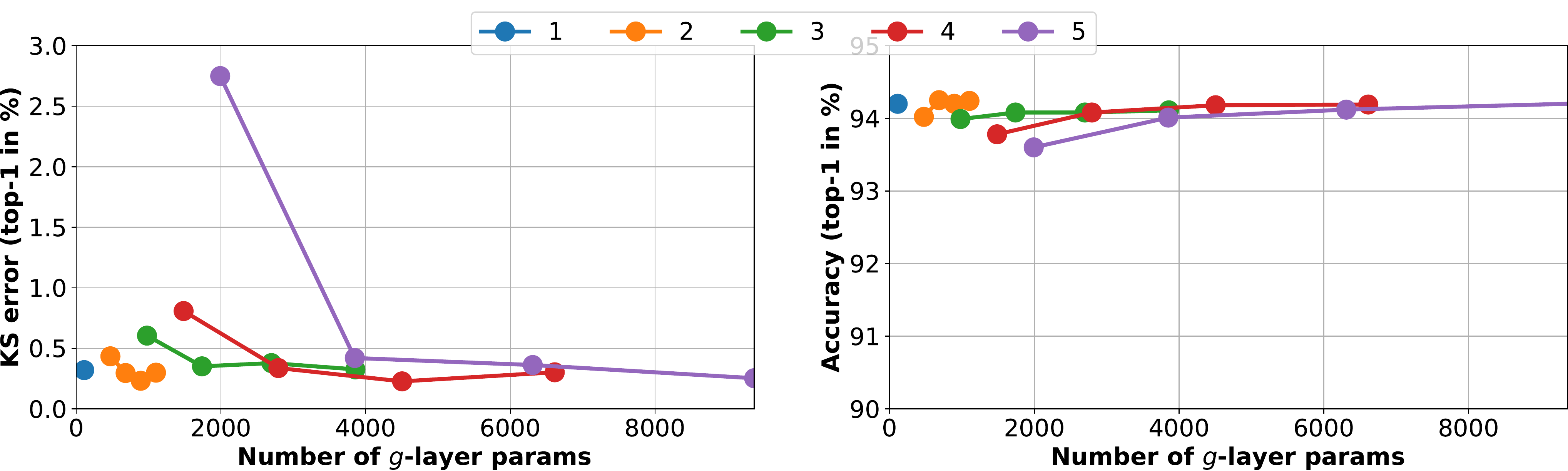}
        \caption{CIFAR-10 -- ResNet Wide 32}
    \end{subfigure}    
    
    \caption{Calibration (\textbf{left}) and Accuracy (\textbf{right}) as function of number of parameters. Number of parameters change by the number of dense layers, indicated by colored lines, and number of dense units per layer: $\{2, 3, 4, 5\}$ units per class + 2.}
    \label{fig:denselayers_denseunits_part1}
\end{figure*} 

\begin{figure*}[p]   
    \begin{subfigure}[b]{\textwidth}
        \centering
        \includegraphics[height=45mm]{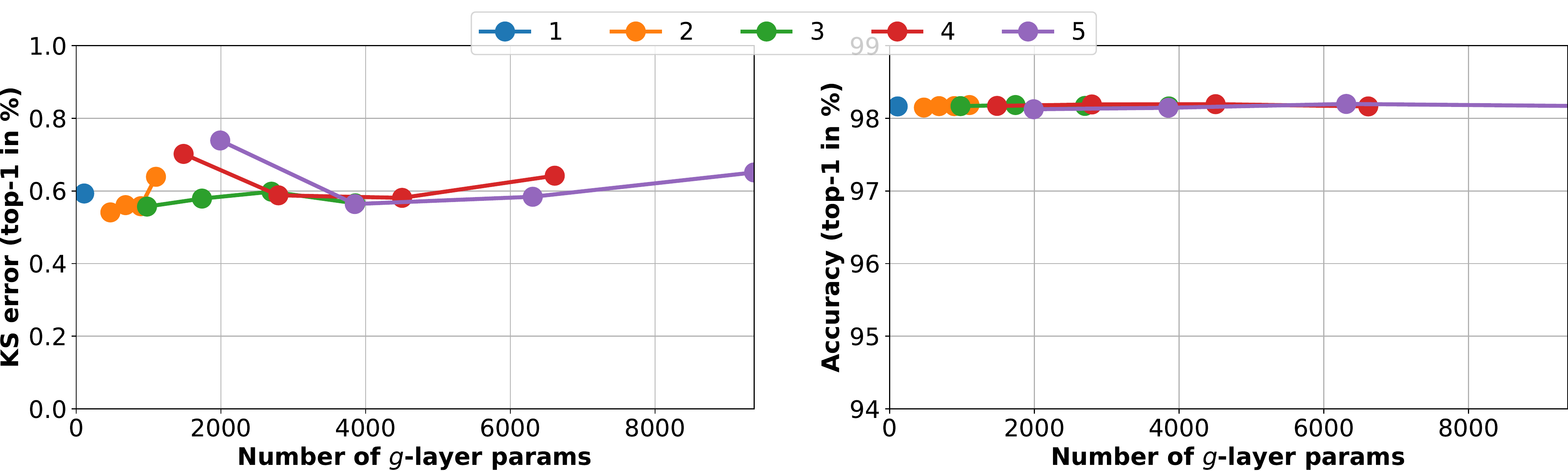}
        \caption{SVHN -- ResNet 152 SD}
    \end{subfigure}\vspace{5mm}
    
    \begin{subfigure}[b]{\textwidth}
        \centering
        \includegraphics[height=45mm]{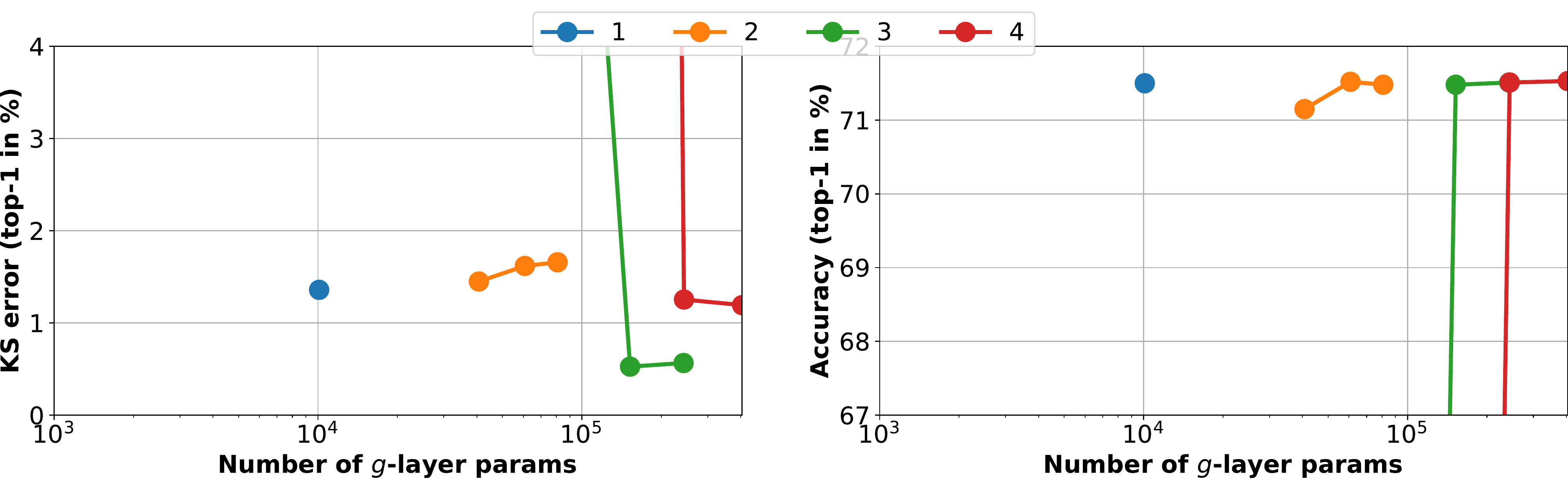}
        \caption{CIFAR-100 -- ResNet 110}
    \end{subfigure}\vspace{5mm}

    \begin{subfigure}[b]{\textwidth}
        \centering
        \includegraphics[height=45mm]{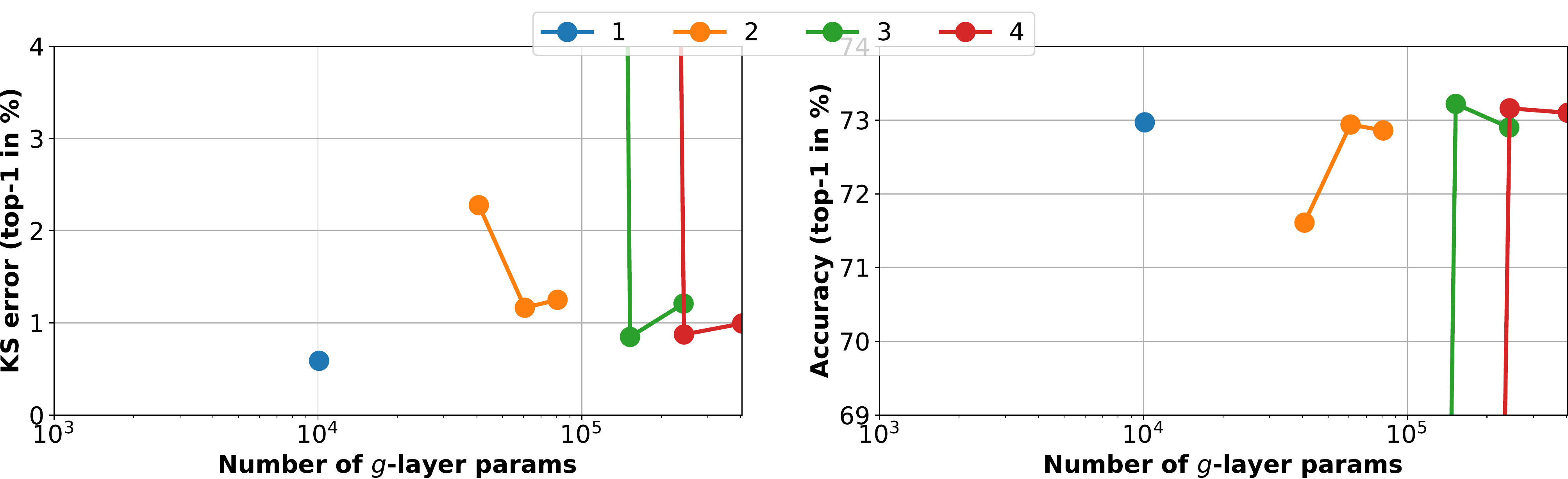}
        \caption{CIFAR-100 -- ResNet 110 SD}
    \end{subfigure}\vspace{5mm}
    
    \begin{subfigure}[b]{\textwidth}
        \centering
        \includegraphics[height=45mm]{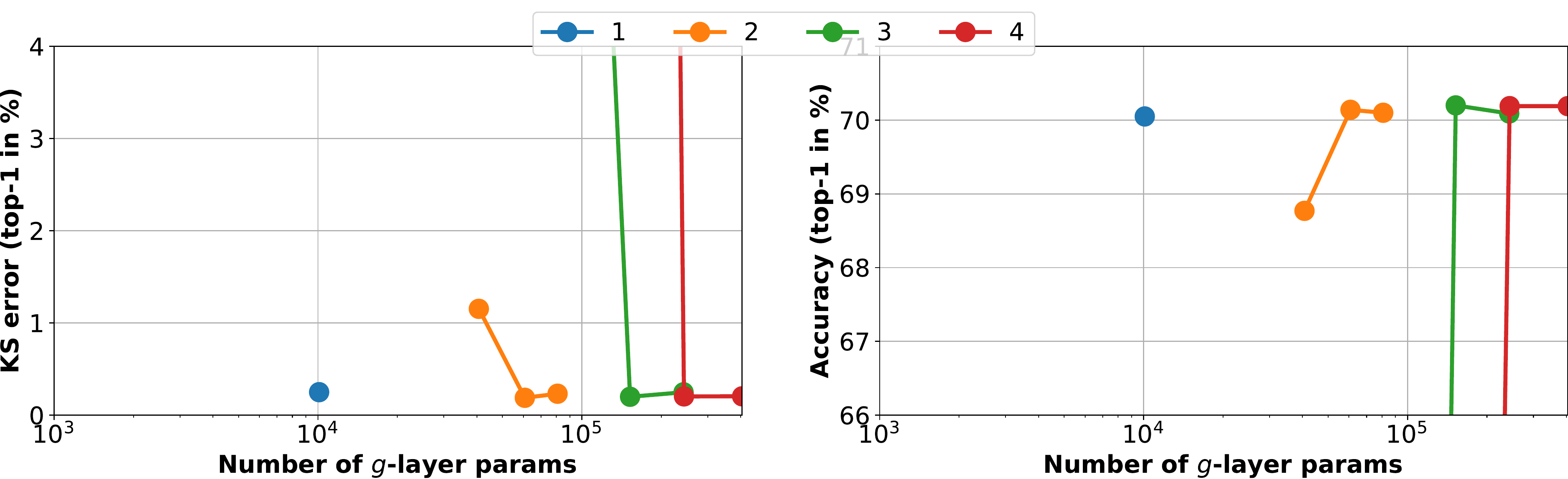}
        \caption{CIFAR-100 -- DenseNet 40}
    \end{subfigure}    
    \caption{Calibration (\textbf{left}) and Accuracy (\textbf{right}) as function of number of parameters. Number of parameters vary by the number of dense layers (colored lines), and number of dense units per layer: $\{2, 3, 4\}$ units per class (SVHN also 5) + 2.}    
    \label{fig:denselayers_denseunits_part2}
\end{figure*} 

\begin{figure*}[t]       
    \begin{subfigure}[b]{\textwidth}
        \centering
        \includegraphics[height=45mm]{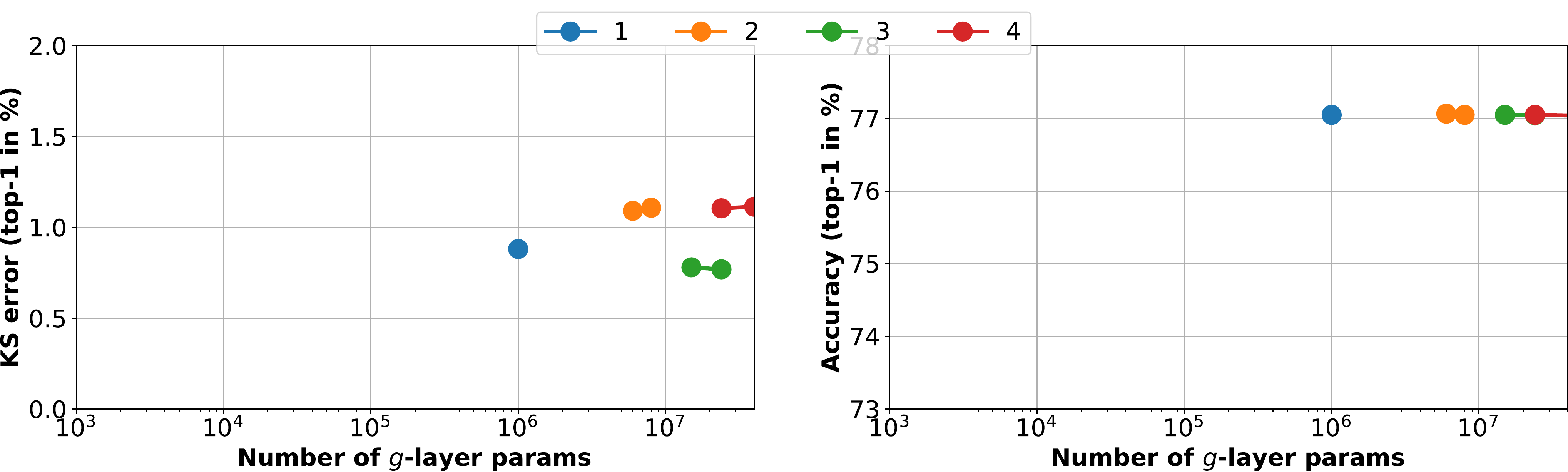}
        \caption{ILSVRC'12 / ImageNet -- DenseNet 161}
    \end{subfigure}\vspace{5mm}
    
    \begin{subfigure}[b]{\textwidth}
        \centering
        \includegraphics[height=45mm]{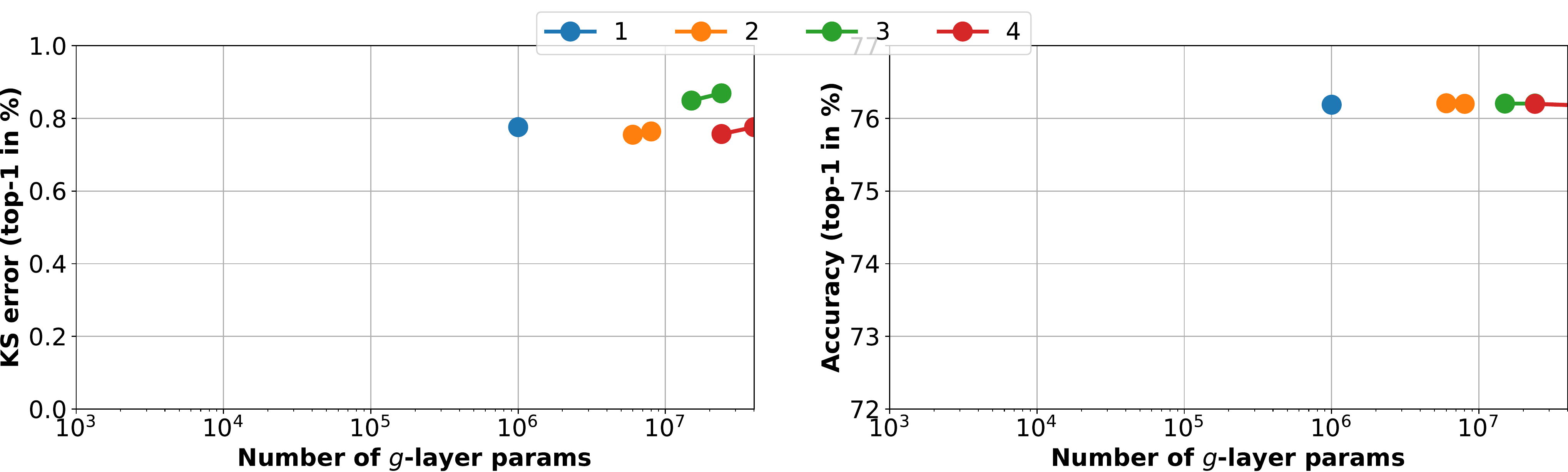}
        \caption{ILSVRC'12 / ImageNet -- ResNet 152}
    \end{subfigure}    
    \caption{Calibration (\textbf{left}) and Accuracy (\textbf{right}) as function of number of parameters. Number of parameters change by the number of dense layers (colors) and number of dense units per layer ($\{3, 4\}$ times number of classes + 2).}
    \label{fig:denselayers_denseunits_part3}
\end{figure*}
In this set of experiments we vary the number of parameters in the $g$-layer networks by changing the number of layers and the number of dense units per class. We measure the KS error (top 1) and the accuracy. The results are in~\fig{denselayers_denseunits_part1}-~\fig{denselayers_denseunits_part3}.

\subsection{Layer Initialisation}
\begin{figure*}[p]
    \centering
    \begin{subfigure}[b]{.85\textwidth}
        \includegraphics[width=\textwidth]{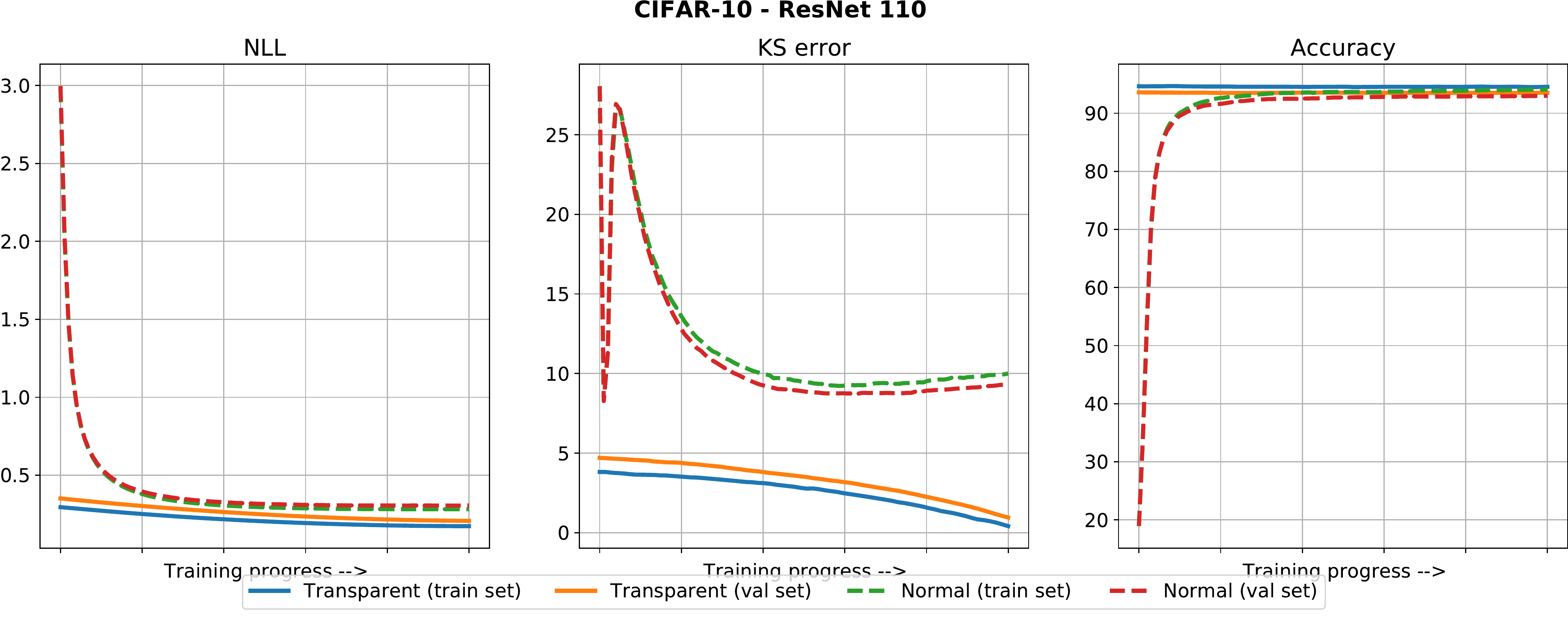}
        \caption{CIFAR-10 - ResNet 110: Transparent vs Normal layer initialisation}
    \end{subfigure}
    \vspace{5mm}
    
    \begin{subfigure}[b]{.85\textwidth}
        \includegraphics[width=\textwidth]{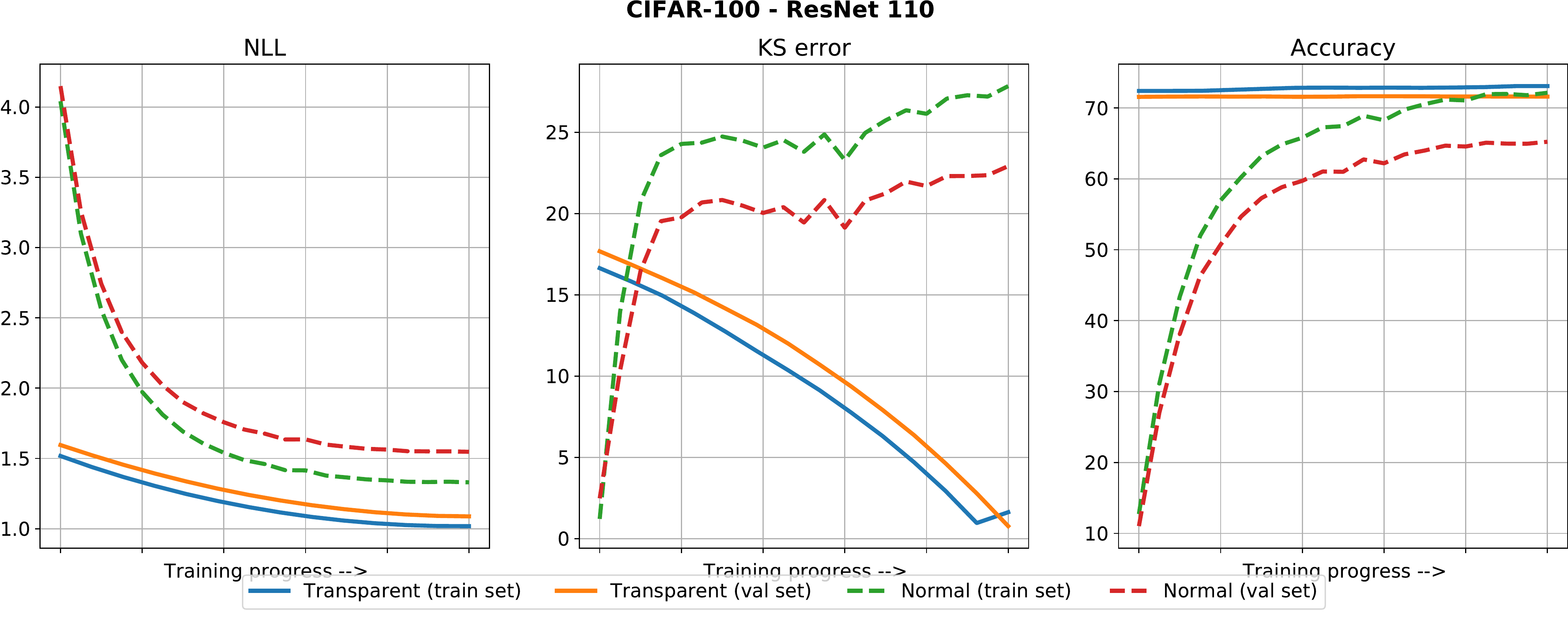}
        \caption{CIFAR-100 - ResNet 110: Transparent vs Normal layer initialisation}
    \end{subfigure}
    \vspace{5mm}
    
    \begin{subfigure}[b]{.85\textwidth}
        \includegraphics[width=\textwidth]{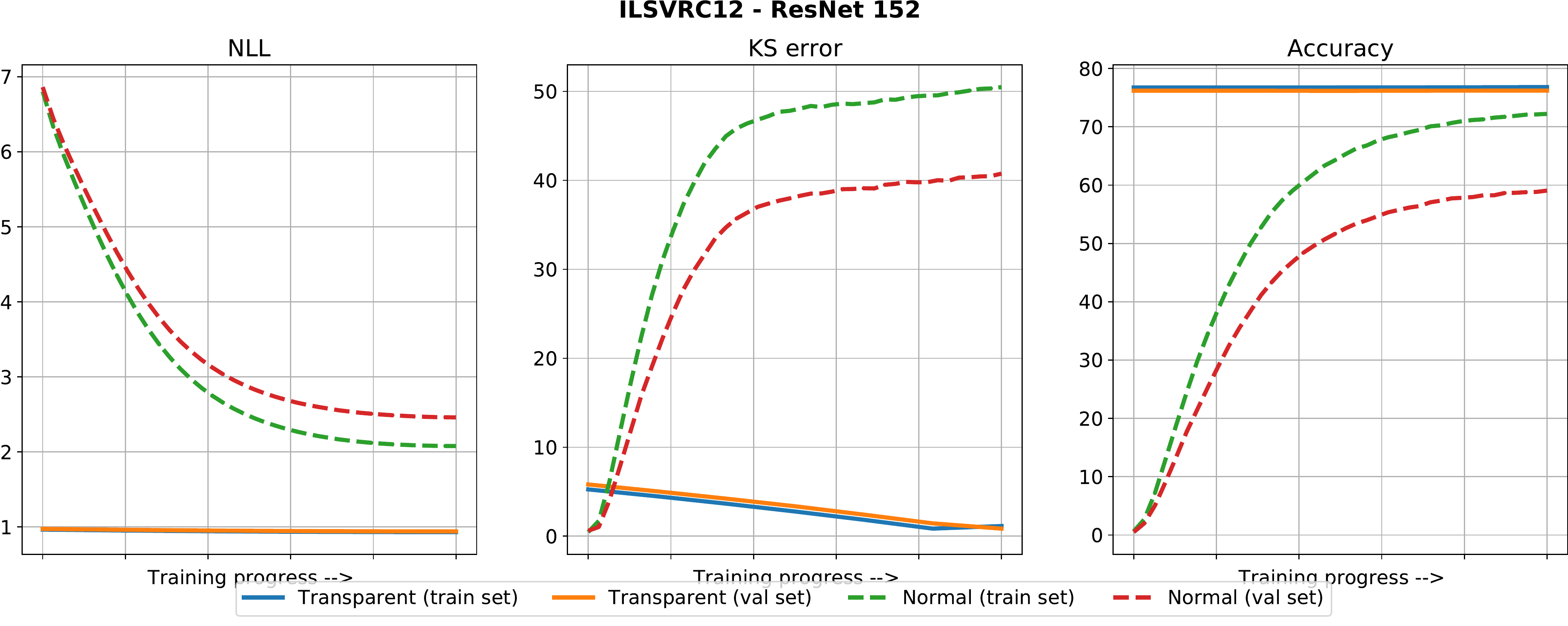}
        \caption{ImageNet - ResNet 152: Transparent vs Normal layer initialisation}
    \end{subfigure}
    
    \caption{Negative log likelihood, KS error and accuracy measured during training of $g$-layer networks either initialised with transparent layers or with normal (random Glorot \cite{glorot10aistats}) initialisation.}
    \label{fig:transparent_initialisation}
\end{figure*}
In this last set of experiments we compare the initialisation of the layers, using transparent initialisation, which ensure an identity transformation, or with standard random initialisation (using Glorot initialisation~\cite{glorot10aistats}).

For this comparison we use ResNet 110 for CIFAR-10 and CIFAR-100 and ResNet 152 for ImageNet, using a 3 $g$-layer network with 3 hidden units per class (resulting in 32, 302, and 3002 hidden units respectively). For each model, after each training epoch, we compute the negative log-likelihood (NLL), KS error and accuracy for the train set and test set. Since our models use early stopping, we re-normalize the x-axis to range from 0 - 100\% training progress (instead of the number of epochs). 
For the transparent model the hyper-parameters from the cross validation search, for the normal (random initialised) models we manually tune the learning rate to get decent performance. \textbf{Note} the goal of this experiment is to show the benefit of transparent initialisation for $g$-layer training, not to get the best accuracy when trained from random initialisation. 

The results are in~\fig{transparent_initialisation}.
From the results we observe that transparent initialisation ensures that the accuracy remains at the same level of the base network. This is in stark contrast with normal (random initialised) models, where the accuracy starts from random performance. 
Subsequently, it seems that the while the normal models are able to learn the correct classification, this comes at the cost of their calibration. 
Hence we conclude that for \emph{calibration} with $g$-layers transparent initialisation is preferred.

\bibliography{glayers_aaai}
\end{document}